\documentclass{article}
\PassOptionsToPackage{square, numbers, compress, sort}{natbib}
\usepackage[final]{neurips_2021}
\usepackage{neurips_2021}

\usepackage[utf8]{inputenc} 
\usepackage{amsbsy}  
\usepackage[T1]{fontenc}    
\usepackage{amsmath,amsfonts,amsthm,amssymb}       
\usepackage[hidelinks]{hyperref}       
\usepackage[capitalise]{cleveref}
\usepackage{color}
\usepackage{url}            
\usepackage{booktabs}       
\usepackage[algo2e,ruled]{algorithm2e}
\usepackage{nicefrac}       
\usepackage{microtype}      
\usepackage[pdftex]{graphicx}
\usepackage{subcaption}
\usepackage{mwe}
\usepackage{times}
\usepackage{titlesec} 
\usepackage{multirow}
\usepackage[inline]{enumitem}
\usepackage{hhline}  
\usepackage{aliascnt} 
\usepackage{todonotes} 
\usepackage{wrapfig}
\usepackage{mathtools} 
\usepackage{bm}
\usepackage{bbm}
\usepackage{tikz}
\usepackage{floatrow}
\usepackage{makecell}
\usepackage{afterpage} 

\usetikzlibrary{bayesnet}

\newtheorem{observation}{Observation}

\newtheorem{theorem}{Theorem}
\newtheorem{assumption}{Assumption}

\newtheorem{lemma}{Lemma}

\newlength\inputlen

\crefname{equation}{Eq.}{Eqs.}
\crefname{figure}{Fig.}{Figs.}
\crefname{section}{Sec.}{Secs.}
\crefname{subsection}{Sec.}{Secs.}
\crefname{theorem}{Thm.}{Thms.}
\crefname{corollary}{Cor.}{Cors.}
\crefname{lemma}{Lemma}{Lemmas}
\crefname{observation}{Observation}{Observations}
\crefname{definition}{Defn.}{Defns.}
\crefname{assumption}{Assumption}{Assumptions}
\crefname{appendix}{Appx.}{Appx.}
\crefname{algocf}{Alg.}{Algs.}
\Crefname{algocf}{Algorithm}{Algorithms}

\newif\ifcomments
\commentstrue
\ifcomments\newcommand{\comments}[1]{#1}\else\newcommand{\comments}[1]{}\fi

\definecolor{clrgp}{rgb}{.9,0,.9}
\definecolor{gray}{rgb}{0.41, 0.41, 0.41}
\definecolor{forestgreen}{rgb}{0.13, 0.55, 0.13}

\titleformat*{\subparagraph}{\normalfont\itshape}
\renewcommand{\paragraph}[1]{\vspace{0.20ex}\noindent\textbf{#1}}
\renewcommand{\subparagraph}[1]{\vspace{0.20ex}\noindent\textit{#1}}

\usepackage{xcolor}
\definecolor{dark-red}{rgb}{0.4,0.15,0.15}
\definecolor{dark-blue}{rgb}{0.15,0.15,0.4}
\definecolor{medium-blue}{rgb}{0,0,0.5}
\hypersetup{
   colorlinks, linkcolor={dark-blue},
   citecolor={dark-blue}, urlcolor={medium-blue}
}

\setlist[enumerate,1]{label={\arabic*)}}

\newcommand{\titl}{The Limitations of Large Width in Neural Networks: A Deep Gaussian Process Perspective}

\newcommand{\authorinfo}{
  Geoff Pleiss \\
  Columbia University \\
  gmp2162@columbia.edu
  \And
  John P. Cunningham \\
  Columbia University \\
  jpc2181@columbia.edu
}

\title{\titl}
\author{\authorinfo}

\DeclareMathOperator*{\expectedvalue}{\mathbb{E}}
\DeclareMathOperator*{\variance}{Var}

\newcommand{\mparen}[1]{\ensuremath{\mathchoice{ \left( #1 \right) }{ (#1) }{ (#1) }{}}}
\newcommand{\mbracket}[1]{\ensuremath{\mathchoice{ \left[ #1 \right] }{ [#1] }{ [#1] }{}}}

\newcommand{\reals}{\ensuremath{\mathbb{R}}}
\newcommand{\GP}[2]{\ensuremath{\mathcal{GP} \mbracket{ #1, #2 }}}
\newcommand{\normaldist}[2]{\ensuremath{\mathcal{N} \mparen{ #1, #2 } }}

\newcommand{\Ev}[1]{\ensuremath{\expectedvalue \mbracket{ #1 }}}
\newcommand{\Evover}[2]{\ensuremath{\expectedvalue_{#1} \mbracket{ #2 }}}
\newcommand{\Var}[1]{\ensuremath{\variance \mbracket{ #1 }}}

\newcommand{\intd}[1]{\,\mathrm{d}{#1}}

\newif\ifboldmatrix
\boldmatrixtrue

\ifboldmatrix\newcommand{\boldmatrix}[1]{\mathbf{#1}}\else\newcommand{\boldmatrix}[1]{#1}\fi
\ifboldmatrix\else\fi
\ifboldmatrix\else\fi

\newcommand{\balpha}{\ensuremath{\boldsymbol{\alpha}}}
\newcommand{\bb}{\ensuremath{\mathbf{b}}}

\newcommand{\bfn}{\ensuremath{\mathbf{f}}}
\newcommand{\bg}{\ensuremath{\mathbf{g}}}

\newcommand{\bk}{\ensuremath{\mathbf{k}}}

\newcommand{\bsigma}{\ensuremath{\boldsymbol{\sigma}}}
\newcommand{\bt}{\ensuremath{\mathbf{t}}}

\newcommand{\bw}{\ensuremath{\mathbf{w}}}

\newcommand{\bx}{\ensuremath{\mathbf{x}}}
\newcommand{\bxi}{\ensuremath{\boldsymbol{\xi}}}
\newcommand{\by}{\ensuremath{\mathbf{y}}}
\newcommand{\bz}{\ensuremath{\mathbf{z}}}
\newcommand{\bzero}{\ensuremath{\mathbf{0}}}

\newcommand{\bF}{\ensuremath{\boldmatrix{F}}}

\newcommand{\bI}{\ensuremath{\boldmatrix{I}}}
\newcommand{\bK}{\ensuremath{\boldmatrix{K}}}
\newcommand{\bL}{\ensuremath{\boldmatrix{L}}}

\newcommand{\bW}{\ensuremath{\boldmatrix{W}}}
\newcommand{\bX}{\ensuremath{\boldmatrix{X}}}

\begin{document}
\maketitle

\begin{abstract}
  Large width limits have been a recent focus of deep learning research: modulo computational practicalities, do wider networks outperform narrower ones?
Answering this question has been challenging, as conventional networks gain representational power with width, potentially masking any negative effects.
Our analysis in this paper decouples capacity and width via the generalization of neural networks to Deep Gaussian Processes (Deep GP),
a class of nonparametric hierarchical models that subsume neural nets.
In doing so, we aim to understand how width affects (standard) neural networks once they have sufficient capacity for a given modeling task.
Our theoretical and empirical results on Deep GP suggest that \emph{large width can be detrimental to hierarchical models}.
Surprisingly, we prove that even nonparametric Deep GP converge to Gaussian processes, effectively becoming shallower without any increase in representational power.
The posterior, which corresponds to a mixture of data-adaptable basis functions, becomes less data-dependent with width.
Our tail analysis demonstrates that width and depth have opposite effects: depth accentuates a model's non-Gaussianity, while width makes models increasingly Gaussian.
We find there is a ``sweet spot'' that maximizes test performance before the limiting GP behavior prevents adaptability,
occurring at $\text{width}=1$ or $\text{width}=2$ for nonparametric Deep GP.
These results make strong predictions about the same phenomenon in conventional neural networks trained with L2 regularization
(analogous to a Gaussian prior on parameters):
we show that such neural networks may need up to $500-1000$ hidden units for sufficient capacity---depending on the dataset---%
but further width degrades performance.

\end{abstract}

\section{Introduction}
\label{sec:intro}

Research has shown that deeper neural networks tend to be more expressive and efficient than wider networks under a variety of metrics \citep[e.g.][]{montufar2014number,romero2014fitnets,telgarsky2016benefits,poole2016exponential,raghu2017expressive,nguyen2017loss,lu2017expressive,chen2020towards}.
Nevertheless, there is resurgent interest in wide models due in part to empirical successes \citep[e.g.][]{zagoruyko2016wide} and theoretical analyses of limiting behavior.
When randomly initialized to create a distribution over functions, neural networks converge to Gaussian processes ({\bf GP}) as width increases.
This result, first proved for 2-layer networks \citep{neal1995bayesian}, has been extended to deeper networks \citep{lee2017deep,matthews2018gaussian}, convolutional networks \citep{garriga2019deep,novak2019bayesian}, and other architectures \cite{yang2019tensor,hron2020infinite}.
A similar limit exists for gradient-trained networks, which behave increasingly like kernel machines under the neural tangent kernel \citep[e.g.][]{jacot2018neural,du2019gradient,allen2019learning,arora2019exact,lee2019wide,geiger2020scaling,yang2020tensor}.

While these limits simplify analyses, there is something unsettling about reducing neural networks to kernel methods.
\citet[][p. 161]{neal1995bayesian} describes the GP limit as ``disappointing,'' noting that
``infinite networks do not have hidden units that represent `hidden features'\ldots
often seen [as the] interesting aspect of neural network learning.''
Recent work indeed shows that
learned hierarchical features can be exponentially more efficient than the fixed shallow representations of kernels
\citep[e.g.][]{arora2019exact,allen2019can,allen2020backward,bai2020beyond,bengio2005curse,chen2020towards,ghorbani2019limitations,ghorbani2020neural,li2020learning,yehudai2019power}.
At the same time, wider networks can more accurately model complex functions \citep{goodfellow2016deep}.
Thus, wide limits appear to confound opposing phenomenon: increased capacity makes them more expressive, yet the loss of hierarchical features seems to make them less expressive.
This may explain the mixed empirical performance of limiting models:
outperforming finite width models in some scenarios \citep[e.g.][]{arora2020harnessing,lee2020finite,garriga2019deep,geiger2020scaling},
yet falling short on more complex tasks \citep[e.g.][]{arora2019exact,bai2020beyond,fort2020deep,shankar2020neural,lee2019wide}.

This paper aims to decouple these effects of large width.
Our goal is to understand the inductive biases of wide networks,
after a network has ``sufficient'' capacity for a given modeling task.  We ask:
\emph{
	If we control for the effects of increased capacity, what---if any---value remains in wide networks?
}

To achieve this control, we note that a typical neural network layer corresponds to a finite basis, where elementwise nonlinearities transform each hidden feature into a \emph{single basis function}.
In order to decouple width from capacity, one could generalize these layers so that each nonlinearity produces any number of basis functions;
if each hidden feature gives rise to an infinite and universal basis, then hidden layers would have infinite representational capacity \emph{regardless of width}.
This generalization is in fact a well-studied class of hierarchical models---Deep Gaussian Processes ({\bf Deep GP}) \citep{damianou2013deep,damianou2015deep,bui2016deep,cutajar2017random,salimbeni2017doubly,havasi2018inference,dutordoir2020bayesian,dutordoir2021gpflux}---%
where standard neural net layers are replaced with vector-valued Gaussian processes.
Indeed, typical neural networks are a degenerate Deep GP subclass \cite{aitchison2020bigger,agrawal2020wide,duvenaud2014avoiding,ober2020global}.

We therefore have a generalization of neural networks where capacity is controlled,
from which we can glean insights about conventional networks that have sufficient representational power for a given modeling task.
Surprisingly, despite using Gaussian processes as the primary hierarchical component, we prove that \emph{Deep GP converge to (single-layer) GP in their infinite width limit} (\cref{thm:dgp_gp}).
Troubling implications immediately ensue:
large width is strictly detrimental to Deep GP, as the limiting model collapses to a shallower version of itself.
We support this theorem with an analysis of neural network and Deep GP posteriors, which \emph{become less adaptable as width increases}.
Specifically, we show that the posterior mean corresponds to a mixture of functions drawn from data-dependent (and thus adaptive) reproducing kernel Hilbert spaces, formalizing the above claim from  \citet{neal1995bayesian}.
As width increases, this mixture collapses to the data-independent kernel of the limiting GP, implying that wider models have less feature learning.
Finally, we present a novel tail analysis which indicates that \emph{width and depth have opposite effects}:
depth accentuates non-Gaussianity, sharpening peaks and fattening tails, whereas width increases Gaussianity (\cref{thm:mean_concentration,thm:tails}).

Our theoretical results hold for Deep GP and conventional (parametric) neural networks alike.
Experiments confirm that---%
after a model achieves sufficient capacity\footnote{
	We offer a formal notion of ``sufficient capacity'' in \cref{sec:dgp_sufficient_capacity}.
}---\emph{width can become harmful to model fit and performance}.
For nonparametric Deep GP, a width of 1 or 2 often achieves the best performance.
Neural networks---because of their parametric nature---naturally require more hidden units before achieving optimal accuracy.
Nevertheless, for Bayesian neural networks and conventional (optimized) neural networks trained with L2 regularization,
performance degrades after a certain width.
On small datasets ($N \leq 1000$) with low dimensionality,
we find that models with $\leq16$ hidden units achieve best test set performance.
On larger datasets like CIFAR10, this ``sweet spot'' occurs later (at $\approx 500$ hidden units for sufficiently deep models), yet performance degrades beyond this width.
We note that these trends do not necessarily hold for models that do not have a probabilistic interpretation---%
i.e. optimized neural networks trained without (or nearly without) L2 regularization.
Nevertheless, our findings suggest that narrower models have better inductive biases,
and wide models perform well \emph{in spite of}---not because of---large width.

\section{Setup}

\subsection{Related Work}
\label{sec:related}

\paragraph{Effects of width.}
Works have shown that, given finite parameters, deeper models are more expressive than wider models \citep{montufar2014number,telgarsky2016benefits,lu2017expressive,poole2016exponential,raghu2017expressive}.
Similarly to our work, \citet{aitchison2020bigger} recognises the link between finite neural networks and Deep GP,
and argues that finite neural networks have flexibility in the top-layer representation that is absent in the infinite-width limit.
\citet{halverson2021neural} draw a connection to quantum field theory to argue that neural networks become ``simpler'' near their infinite-width limit.
In the non-probabilistic setting, it is worth noting that wide models have been shown to have favorable optimization landscapes \citep{nguyen2017loss,li2018visualizing,soltanolkotabi2018theoretical,arora2019fine,du2019gradient}
and are resistant to overfitting via double descent \citep{belkin2019reconciling,nakkiran2020deep,cao2020generalization}.
Our work controls for these factors by examining nonparametric hierarchical models with exact Bayesian inference,
and thus does not disagree with these other works.
Infinite width limits have received renewed interest
in Bayesian \citep{neal1995bayesian,lee2019wide,matthews2018gaussian,garriga2019deep,novak2019bayesian,yang2019tensor,hron2020infinite}
and non-Bayesian \citep{jacot2018neural,du2019gradient,allen2019learning,arora2019exact,lee2019wide,yang2020tensor,chizat2018global,mei2018mean,golikov2020towards,yang2020feature} settings.
Most of these works show that neural networks converge to kernel methods,
though recent work suggests that  this limiting behavior can be
avoided with different parameterizations \citep[e.g.][]{chizat2018global,mei2018mean,golikov2020towards,yang2020feature}.
Similarly to \citet{lee2019wide}, our Deep GP limit analysis sequentially increases the width of each layer,
though we hypothesize a similar proof exists where the width of all layers increases simultaneously (akin to \citep{matthews2018gaussian}).

\paragraph{Deep GP}
are introduced by \citet{damianou2013deep}.
A large portion of Deep GP research has thus far focused on scalable approximate inference methods
\citep{bui2016deep,dai2016variational,wang2016sequential,cutajar2017random,salimbeni2017doubly,havasi2018inference,dutordoir2021gpflux,ober2020global}.
Though prior work has studied tail properties of neural networks \citep{vladimirova2019understanding,zavatoneveth2021exact} and Deep GP with RBF kernels \citep{lu2020interpretable},
our work is---to the best of our knowlege---the first general result for Deep GP tails.
\citet{duvenaud2014avoiding} and \citet{dunlop2018deep} investigate pathological behaviors that arise with depth,
while \citet{agrawal2020wide} note that ``bottlenecked'' Deep GP have better performance and correlations among predictive tasks.
Our work complements these analysis by characterizing the effects of width.

\paragraph{Connections between Deep GP and neural networks.}
Many researchers have noted connections between neural networks and Deep GP \citep[e.g.][]{gal2016dropout,louizos2016structured,cutajar2017random,dutordoir2021deep}.
\citet{duvenaud2014avoiding} suggest that infinitely-wide neural networks with intermediate bottleneck layers are nonparametric Deep GP.
\citet{agrawal2020wide} formalize this connection,
but note that not all Deep GP can be constructed from bottlenecked neural networks
(see \cref{sec:proof_dgp_gp}).
In contrast to these prior works, we avoid reducing Deep GP to neural networks,
and instead reduce neural networks to degenerate Deep GP.

\subsection{A Covariance Perspective on Gaussian Process Limiting Behavior}
\label{sec:kernelized_nn}

To decouple the effects of increasing width and capacity, we first prove a new result about GP limits for a more general class of models, including Deep GP as well as typical neural networks. This result forms a necessary foundation for the subsequent theorems that are a main contribution of this work.
To begin, note that the proof technique introduced by \citet{neal1995bayesian} and extended by others \citep{lee2017deep,matthews2018gaussian,garriga2019deep,novak2019bayesian,yang2019tensor,hron2020infinite}
relies on the multivariate central limit theorem, which requires a model with additive structure.
Deep GP do not generally decompose in an additive manner, so we establish a more general proof technique.
For simplicity, we first present it in the context of neural networks, and then extend it to a more general class of models.

Consider the 2-layer neural network $f_2(\bfn_1(\bx))$, with
$\bfn_1: \reals^D \to \reals^{H_1}$ and $f_2 : \reals^{H_1} \to \reals$:
\begin{equation}
  \bfn_1(\cdot) = \bW_1^\top (\cdot) + \beta \bb_1, \qquad f_2(\cdot) = {\textstyle \frac{1}{\sqrt{H_1}}} \bw_2^\top \bsigma(\cdot) + \beta b_2.
  \label{eqn:2layer_nn}
\end{equation}
$\bsigma(\cdot)$ is an elementwise nonlinearity,
$\beta$ is a positive constant,
and $\bW_1$, $\bb_1$, $\bw_2$, and $b_2$ are i.i.d. Normal.
With randomly initialized parameters, $f_2(\bfn_1(\cdot)): \reals^D \to \reals$ is a prior distribution over functions,
and this distribution converges to a GP in the infinite width limit \citep{neal1995bayesian}.

\begin{lemma}\label{lemma:nn_cf}
  \vspace{0.5em}
  The neural network defined in \cref{eqn:2layer_nn} is a Gaussian process if and only if---%
  for any finite set of inputs $\bX = [\bx_1, \ldots, \bx_N]$---%
  the conditional prior covariance $\Evover{\bfn_2 \mid \bX, \bW_1, \bb_1}{ \bfn_2 \bfn_2^\top }$
  is almost surely equal to the marginal prior covariance $\Evover{\bfn_2 \mid \bX}{ \bfn_2 \bfn_2^\top }$,
  where $\bfn_2 \mid \bX \triangleq [f_2(\bfn_1(\bx_1)), \ldots, f_2(\bfn_1(\bx_N))]$.
\end{lemma}
\begin{proof}
  By definition, $f_2(\bfn_1(\cdot))$ is a GP if and only if
  $\bfn_2 \mid \bX$ is multivariate Gaussian for any $\bX$.
  From \cref{eqn:2layer_nn}, we have
  $
  p(\bfn_2 \mid \bX, \bW_1, \bb_1) = \normaldist{\bzero}{\bK_{\bW_1, \bb_1}(\bX,\bX)},
  $
  where $[\bK_{\bW_1, \bb_1}(\bX,\bX)]_{ij} = \beta^2 + \frac{1}{H_1} \bsigma(\bW_1^\top \bx_i + \beta \bb_1)^\top \bsigma(\bW_1^\top \bx_j + \beta \bb_1)$ is the appropriate kernel Gram matrix.
  Using Jensen's inequality, we have a lower bound on the characteristic function of $\bfn_2 \mid \bX$:
  \begin{align}
    \Evover{\bfn_2 \mid \bX}{ \exp \left( i \bt^\top \bfn_2 \right) }
    &=
    \Evover{\bW_1, \bb_1}{
      \Evover{\bfn_2 \mid \bX, \bW_1, \bb_1}{ \exp \left( i \bt^\top \bfn_2 \right) }
    }
    \tag{law of total expectation}
    \\
    &=
    \mathbb{E}_{\bW_1, \bb_1} \biggl[
      \exp \left( -\tfrac 1 2 \bt^\top \bK_{\bW_1, \bb_1}(\bX,\bX) \bt \right)
    \biggr]
    \tag{char. func. of a Gaussian}
    \\
    &\geq \exp \left(
      -\tfrac 1 2 \bt^\top \Evover{ \bW_1, \bb_1 }{ \bK_{\bW_1, \bb_1}(\bX,\bX)
      } \bt
    \right).
    \tag{convexity of $\exp$}
  \end{align}
  This lower bound happens to be the characteristic function
  of $\normaldist{\bzero}{\Evover{\bW_1, \bb_1}{ \bK_{\bW_1, \bb_1}(\bX, \bX) }}.$
  Since $\exp$ is strictly convex, the characteristic function of $\bfn_2 \mid \bX$ equals the Gaussian lower bound $\forall \bt$
  if and only if $p(\bK_{\bW_1, \bb_1}(\bX,\bX) \mid \bW_1, \bb_1) = \Evover{\bfn \mid \bX, \bW_1, \bb_1}{\bfn_2 \bfn_2^\top}$ is a constant with probability 1.
\end{proof}
Seeing that $\frac{1}{H_1} \bsigma(\bW_1^\top \bx_i + \beta \bb_1)^\top \bsigma(\bW_1^\top \bx_j + \beta \bb_1)$
becomes a.s. constant as $H_1 \to \infty$,
\cref{lemma:nn_cf} re-establishes the result of \citet{neal1995bayesian} (see \cref{sec:additive_dgp_proof}).
Critically, unlike \citeauthor{neal1995bayesian}'s proof, \cref{lemma:nn_cf} neither relies on the central limit theorem nor requires $f_2(\bfn_1(\cdot))$ to be a neural network;
it holds if $p(\bfn_2 \mid \bfn_1(\bx_1), \ldots, \bfn_1(\bx_N))$ is Gaussian.
Therefore, we can generalize it to a larger class of models:
\begin{lemma}\label{lemma:dgp_cf}
  \vspace{0.5em}
  Let $f_2(\bfn_1(\cdot)) : \reals^D \to \reals$ be a hierarchical model where $f_2(\cdot) : \reals^{H_1} \to \reals$ is a GP and $\bfn_1(\cdot) : \reals^D \to \reals^{H_1}$ is a random vector-valued function
  (including a multilayer hierarchical model).
  Then $f_2(\bfn_1(\cdot))$ is a GP if and only if
  $\Evover{\bfn_2 \mid \bX, \bfn_1(\cdot)}{ \bfn_2 \bfn_2^\top } = \Evover{\bfn_2 \mid \bX}{ \bfn_2 \bfn_2^\top }$
  a.s. for all $\bX = [\bx_1, \ldots, \bx_N]$.
\end{lemma}

The covariance perspective from \cref{lemma:nn_cf,lemma:dgp_cf} is revealing about GP limits.
As $\Evover{\bfn_2 \mid \bX, \bfn_1(\cdot)}{ \bfn_2 \bfn_2^\top }$ converges to $\Evover{\bfn_2 \mid \bX}{ \bfn_2 \bfn_2^\top }$
the model output becomes less and less dependent on $\bfn_1(\cdot)$.
In other words, $f_2(\bfn_1(\cdot))$ \emph{loses its hierarchical nature}.
We reiterate that \cref{lemma:dgp_cf} has no requirements about $f_2(\bfn_1(\cdot))$ transitioning from a finite to infinite basis,
nor %
does it require $f_2(\bfn_1(\cdot))$ to have additive structure.
We demonstrate its generality in the next section with surprising---and troubling---implications.

\section{Deep Gaussian Processes Collapse to Shallow Gaussian Processes}
\label{sec:dgp}
Deep GP \citep{damianou2013deep,damianou2015deep,bui2016deep,salimbeni2017doubly} are hierarchical models where layers $\bfn_1(\cdot) \ldots f_L(\cdot)$ are (vector-valued) GP:
\begin{equation}
  \text{DGP}(\bx) = f_{L} \circ \ldots \circ \bfn_{1} \left( \bx \right),
  \quad
  \bfn_{i}(\cdot) = [f_{i}^{(1)}( \cdot ), \ldots, f_{i}^{(H_i)}( \cdot ) ],
  \quad
  f_{i}^{(j)}( \cdot ) \stackrel{\text{i.i.d}}{\sim} \GP{ 0 }{ k_i(\cdot, \cdot) }.
  \label{eqn:dgp_def}
\end{equation}
$H_i$ is the width of the $i^\text{th}$ GP layer, and the output dimensions of each $\bfn_i(\cdot)$ are independent.
By using GP as the primary hierarchical building blocks, Deep GP are generally nonparametric and, assuming the GP layers use universal kernels \citep{micchelli2006universal}, have infinite representational capacity (see \cref{sec:dgp_capacity}).

\paragraph{Deep GP versus GP.}
Deep GP seek to offer more expressivity: conventional single-layer GP---though also nonparametric---are inherently limited by the choice of the prior covariance function \cite{bui2016deep,salimbeni2017doubly}.
For example, a GP with a RBF covariance is not suitable for data with discontinuities or sharp changes.
However, stacking two RBF GP together---$f_2(\bfn_1(\cdot))$---can overcome this limitation, since $\bfn_1(\bx)$ can encode a warping of $\bx$ that ``smoothes'' the input data for $f_2(\cdot)$ (as we will show in \cref{fig:adaptive_vs_gaussian_rbf}).
Empirically, Deep GP have been shown to offer much more accurate predictive posteriors than standard GP \citep[e.g.][]{damianou2013deep,damianou2015deep,cutajar2017random,salimbeni2017doubly,havasi2018inference,blomqvist2019deep,dutordoir2020bayesian}.

\paragraph{Deep GP versus neural networks.}
(Bayesian) feed-forward neural networks are a strict subclass of Deep GP, albeit a degenerate one \cite{aitchison2020bigger,louizos2016structured,ober2020global}.
The first neural network layer is a GP with a linear kernel, while subsequent layers are GP with the kernel
$k(\bz, \bz') = \beta^2 + \frac{1}{H_{i-1}} \sum_{i=1}^{H_{i-1}} \sigma( z_i ) \: \sigma( z_i )$.
A neural network, unlike other Deep GP, does not have infinite capacity.
Put loosely, a single neural network hidden unit corresponds to a single basis,
while in general a single Deep GP unit corresponds to a potentially-infinite basis.
See \citep{agrawal2020wide,aitchison2020bigger,cutajar2017random,duvenaud2014avoiding,dutordoir2021deep,ober2020global} and \cref{sec:dgp_degenerate} for more discussion on this connection.
The critical takeaway is that all of our Deep GP results apply to neural networks as well.

\subsection{Wide Deep GP are Gaussian Processes}
Having established a model where width does not effect capacity,
we now establish what remaining effects width has.
Empirical evidence suggests that the choice of width impacts Deep GP predictions \citep{bui2016deep,havasi2018inference}.
In practice it is common to make Deep GP as wide as comparably-sized neural networks;
\citet{salimbeni2017doubly} for example train Deep GP with $\geq30$ units per layer.

\emph{Surprisingly, here we prove that---in the limit of infinite width---Deep GP collapse to single-layer Gaussian processes.}
Our proof relies on the conditional covariance analysis of the previous section. 
If the GP layers have non-pathological covariance functions\footnote{
  Any textbook kernel (isotropic kernels, dot product kernels, etc.)
  or any covariance function with a Fourier-Steiljes representation is ``non-pathological;''
  see \cref{sec:assumptions} for formal assumptions.
}---the Deep GP conditional covariance becomes almost surely constant with width (see \cref{lemma:dgp_cond_covar}, \cref{sec:proof_dgp_gp}).
Combining this with \cref{lemma:dgp_cf}:
\begin{theorem}
  \vspace{0.5em}
  Let $f_{L} \circ \ldots \circ \bfn_{1} \left( \bx \right)$ be a zero-mean Deep GP (Eq.~\ref{eqn:dgp_def}),
  where each layer satisfies \cref{ass:kernels_harmonizable,ass:kernels_dimensionality} (non-pathological prior covariances that scale with dimensionality---see \cref{sec:assumptions}).
  Then $\lim_{H_{L-1} \to \infty} \cdots \lim_{H_1 \to \infty} f_{L} \circ \ldots \circ \bfn_{1} \left( \bx \right)$
  converges in distribution to a (single-layer) GP.
  \label{thm:dgp_gp}
\end{theorem}
(See \cref{sec:proof_dgp_gp} for proof.)
The implications of \cref{thm:dgp_gp} are paradoxical and unsettling.
Deep GP are motivated as a more powerful model than standard GP.
However, as we make the model wider, we arrive back where we started---a Gaussian process (although one with a different prior covariance).

A neural network gains representational power in its GP limit, transitioning from a finite-basis model to a nonparametric model.
The Deep GP limit on the other hand has no additional representational power, since Deep GP are already universal approximators at any width.
(Indeed this fact motivates their use as a control.)
The only difference between finite and infinite width Deep GP is the prior distribution itself:
transitioning from non-Gaussian to Gaussian with increasing width.
In the next section, we investigate how this transition affects model performance.

\section{Large Width Limits the Adaptability of Hierarchical Posteriors}
\label{sec:posterior}
Even with \cref{thm:dgp_gp} and its troubling suggestions, it is not immediately clear exactly what is lost in the infinite-width limit.
Here, we quantify specific differences in the predictive capabilities of narrow versus wide models.
In particular, we analyze Deep GP/neural network posterior distributions, rather than focusing on a single model trained through optimization.
We show that these posteriors correspond to a mixture of \emph{data-dependent adaptable bases};
however, as width increases this mixture collapses to the (data-independent) basis of the limiting GP.
This result formalizes the often vague notion of \emph{feature learning}, and demonstrates that it is indeed lost in kernel limits.

\begin{figure}
  \centering
  \includegraphics[width=0.95\linewidth]{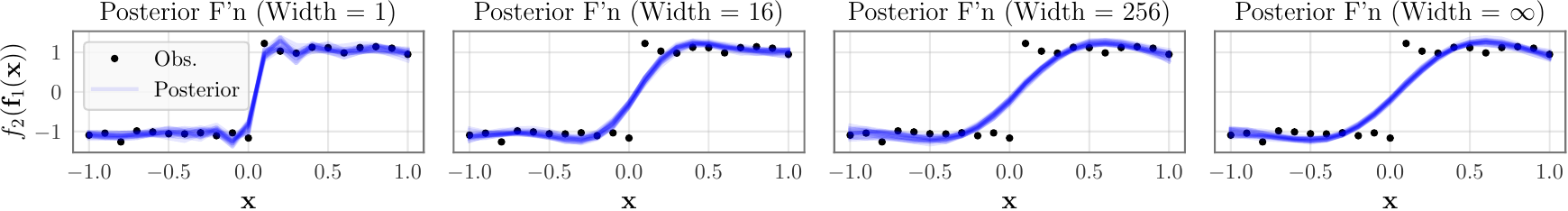}\\
  \vspace{0.2em}
  \includegraphics[width=0.95\linewidth]{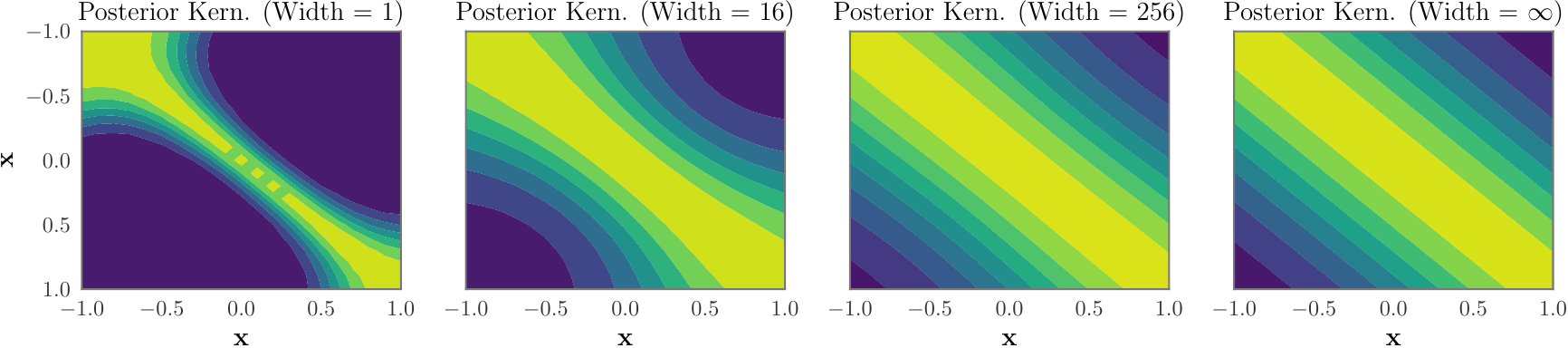}
  \caption{
    {\bf Top:} Posterior of 2-layer RBF Deep GP fit to a noisy step function.
    A width-1 Deep GP fits the discontinuity at $\bx = 0$.
    As width increases, the Deep GP converges to a GP with a stationary covariance unable to fit the step.
    {\bf Bottom:} Average posterior covariance $\Evover{\bfn_1(\bx), \bfn_1(\bx') \mid \by}{ k_2(\bfn_1(\bx), \bfn_1(\bx')) }$.
    The $\text{width} = 1$ posterior covariance is non-stationary, with little covariance around $\bx = 0$.
    As width increases, the posterior covariance becomes stationary (as seen by the kernel's constant diagonals).
  }
  \label{fig:adaptive_vs_gaussian_rbf}
\end{figure}

\paragraph{Hierarchical posteriors correspond to a data-adaptable bases.}
Consider the (finite-width) 2-layer Deep GP $f_2(\bfn_1(\cdot))$, where $k_1(\cdot, \cdot)$ and $k_2(\cdot, \cdot)$ are the covariance functions of $\bfn_1(\cdot)$ and $f_2(\cdot)$.
Given training data $\bX, \by$, define $\bF_1 \triangleq[\bfn_1(\bx_1), \ldots, \bfn_1(\bx_N)]$
and $\bfn_2 \triangleq [f_2(\bfn_1(\bx_1)), \ldots, f_2(\bfn_1(\bx_N))]$.
Let $\bx^*$ be a test input, and let $\bfn_1^*$ and $f_2^*$ equal $\bfn_1(\bx^*)$ and $f_2(\bfn_1(\bx^*))$
(see \cref{fig:2layer_model} in \cref{sec:dgp_posterior_factorization} for a graphical model).
Crucially, $\bfn_2$ and $f^*_2$ only depend on $\bF_1$ and $\bfn^*_1$ through the covariances $\bK_2(\bF_1, \bF_1)$, $\bk_2(\bF_1, \bfn^*_1)$, and $k_2(\bfn^*_1, \bfn^*_1)$
(which we abbreviate as $\bK_2$, $\bk_2^*$, and $k_2^{**}$):
\begin{align}
  p(\bfn_2 \mid \bK_2)
  \sim \normaldist{\bzero}{\bK_2},
  \quad
  p(f^*_2 \mid k^{**}_2, \bk^*_2, \bK_2, \bfn_2)
  \sim \normaldist{\bk_2^{*\top} \bK_2^{-1} \bfn_2}{\:\: k_2^{**} - \bk_2^{*\top} \bK_2^{-1} \bk_2^*},
  \nonumber
\end{align}
By D-separation \citep[e.g.][Ch. 8]{bishop2006pattern}, we can factorize the posterior distribution as:
\begin{align}
  p(f_2^*, \bfn_2, \bK_2, \bk_2^*, k_2^{**} \mid \by) \: = \:
  p(f_2^* \mid \bfn_2, \bK_2, \bk_2^*, k_2^{**}) \:
  p(\bfn_2 \mid \bK_2, \by) \:
  p(\bK_2, \bk_2^*, k_2^{**} \mid \by).
  \label{eqn:2layer_dgp_factorization}
\end{align}
See derivation in \cref{sec:dgp_posterior_factorization}.
Applying the factorization in \cref{eqn:2layer_dgp_factorization}, the posterior mean is:
\begin{align}
  \Evover{f_2^* \mid \by}{ f_2^* }
  &= \Evover{\bK_2, \bk_2^* \mid \by}{
    \Evover{ \bfn_2 \mid \bK_2, \by }{
      \bk_2^{*\top} \bK_2^{-1} \bfn_2
    }
  }
  = \expectedvalue_{ \bK_2, \bk_2^* \mid \by} \Bigl[ \:
    \bk_2^{*\top}
    \overbracket{
      \bK_2^{-1} \Evover{ \bfn_2 \mid \bK_2, \by }{ \bfn_2 }
    }^{ \balpha }
  \: \Bigr]
  \label{eqn:dgp_posterior_mean_kernel} \\
  &=
  \Evover{\bfn_1(\bx^*), \bfn_1(\bx_1), \ldots, \bfn_1(\bx_N) \mid \by}{ \textstyle \sum_{i=1}^N \: \alpha_i \: k_2(\bfn_1(\bx_i), \bfn_1(\bx^*))},
  \label{eqn:dgp_posterior_mean}
\end{align}
where the second line follows from $\bK_2$ and $\bk_2^*$ being deterministic given $\bfn_1(\bx^*)$, $\bfn_1(\bx_1)$, $\ldots$, $\bfn_1(\bx_N)$.
The term inside the \cref{eqn:dgp_posterior_mean} expectation is a function from the reproducing kernel Hilbert space (RKHS) defined by $k_2(\bfn_1(\cdot), \bfn_1(\cdot))$.
We can thus interpret this expectation as an infinite mixture of functions from different Hilbert spaces.
Because the mixture distribution $p(\bfn_1(\bx^*), \bfn_1(\bx_1), \ldots, \bfn_1(\bx_N) \mid \by)$ depends on $\by$, \cref{eqn:dgp_posterior_mean} is an \emph{adaptive data-dependent mixture of RKHS}.

\paragraph{Adaptability is lost in the Gaussian process limit.}
What happens to \cref{eqn:dgp_posterior_mean} as $f_2(\bfn_1(\cdot))$ becomes a Gaussian process in the limit of infinite-width?
Recall from \cref{lemma:dgp_cf} that
the conditional prior covariance becomes deterministic as $f_2(\bfn_1(\cdot))$ converges to a GP.
In other words, the prior and posterior distributions over $\bK_2$ and $\bk_2^*$ become atomic:
$ p(\bK_2, \bk_2^*) =  p(\bK_2, \bk_2^* \mid \by) = \delta \:[ \: \bK_{\text{lim}}, \bk_\text{lim}^* \: ], $
where $\bK_{\text{lim}}$ and $\bk_{\text{lim}}^*$ are shorthand for $\Ev{ \bfn_2 \bfn_2^\top }$ and $\Ev{ \bfn_2 f_2^*}$ respectively.
\cref{eqn:dgp_posterior_mean_kernel} thus collapses to:
\begin{align}
  \lim_{H_1 \to \infty} \: \Evover{f_2^* \mid \by}{ f_2^* }
  =
  \Evover{ \delta [ \bK_\text{lim}, \bk^*_\text{lim} ] }{ \:
    \bk_2^{*\top} \balpha
  \: }
  = {\textstyle \sum_{i=1}^N} \: \alpha_i \: k_{\text{lim}} ( \bx_i, \bx^* ),
  \label{eqn:limiting_dgp_posterior_mean}
\end{align}
which is no longer a mixture of functions from different RKHS.
It is instead a function from a single RKHS (that of the limiting GP prior).\footnote{
  To rigorously argue that the infinite-width posterior collapses in this way,
  we can invoke Proposition 1 from \citet{hron2020exact}.
  See \cref{sec:dgp_posterior_collapse} for details.
}
In other words, while Deep GP (and neural networks) perform \emph{kernel learning (or feature learning)} to adapt to training data,
this ability is lost with large width.

\paragraph{Example.}
Consider a Deep GP with RBF covariances
$
  k_1( \bx, \bx' ) = \exp \left( {- \Vert \bx - \bx' \Vert^2/(2 D) } \right)
$
and
$
  k_2( \bfn_1(\bx), \bfn_1(\bx') ) = \exp \left( - \Vert \bfn_1(\bx) - \bfn_1(\bx') \Vert^2 / (2 H_1) \right).
$
As we show in \cref{sec:limiting_kernels}, this Deep GP converges to a GP with 
$
  k_{\text{lim}}(\bx, \bx') = \exp ( \exp( -\Vert \bx - \bx' \Vert^2 / (2 D) ) - 1 ).
$
Note that this limiting covariance is \emph{stationary}
and is ill-equipped to model the data step in \cref{fig:adaptive_vs_gaussian_rbf}.
However, because $\bfn_1(\cdot)$ is nonlinear, $k_2(\bfn_1(\bx), \bfn_1(\bx'))$
is \emph{nonstationary}.
\cref{fig:adaptive_vs_gaussian_rbf} (top left) shows that the width-1 Deep GP posterior accurately models this data.
The posterior covariance $\Evover{\bfn_1(\bx), \bfn_1(\bx') \mid \by}{ k_2(\bfn_1(\bx), \bfn_1(\bx')) }$ (bottom left) features long-range correlations near $\bx = \pm1$ and short-range correlations near $\bx = 0$.
As width increases, we lose this nonstationarity and the posterior becomes a worse fit.

\section{The Difference Between Width and Depth: A Tail Analysis}
\label{sec:tails}
Our work so far has troubling implications for large width.
On the other hand, empirical evidence has shown that depth improves Deep GP performance---as it does for neural nets \citep[e.g.][]{salimbeni2017doubly,havasi2018inference,ober2020global} (though pathologies can emerge \cite{dunlop2018deep,duvenaud2014avoiding}).
Through a novel tail analysis, we show that width makes Deep GP priors more Gaussian, while depth makes them less Gaussian.
In other words, \emph{width and depth have opposite effects on Deep GP tails}, results that again also apply to typical neural networks.

\paragraph{Deep GP/neural networks are sharply peaked and heavy tailed.}
The proof technique used in \cref{lemma:nn_cf} 
can be used to similarly bound the moment generating function of Deep GP marginals:
\begin{equation}
  \Evover{\bfn_2}{ e^{\bt^\top \bfn_2} } =
  \Evover{\bF_1}{ \Evover{\bfn_2 \mid \bF_1}{ e^{\bt^\top \bfn_2} }}
  \geq
  \exp \left( \frac 1 2 \bt^\top \Evover{\bF_1}{ \bK_2(\bF_1, \bF_1) } \bt \right)
  =
  \Evover{ \bg \sim \normaldist{\bzero}{\bK_\text{lim}}}{ e^{\bt^\top \bg} },
  \label{eqn:mgf_jensen}
\end{equation}
where $\bK_\text{lim} \! = \! \Evover{\bfn_2}{ \bfn_2 \bfn_2^\top } \! = \! \Evover{\bF_1}{\bK_2(\bF_1, \! \bF_1)}$.
Generalizing these bounds to deeper models, we have:
\begin{theorem}
  Let $f_L \circ \ldots \circ \bfn_1(\cdot)$ be a zero-mean Deep GP.
  Given a finite set of inputs $\bX = [\bx_1, \ldots, \bx_N]$, define
  $\bfn_\ell = [(f_\ell \circ \ldots \circ \bfn_1(\bx_1)), \ldots, (f_\ell \circ \ldots \circ \bfn_1(\bx_N))]$ for $\ell \in [1, L]$,
  and define $\bK_\text{lim} = \Evover{\bfn_L}{ \bfn_L \bfn_L^\top }$.
  Then,
  $
  p( \bfn_L = \bzero) \geq \normaldist{\bg = \bzero; \bzero}{ \bK_\text{lim} }.
  $
  \label{thm:mean_concentration}
\end{theorem}
\begin{theorem}
  \label{thm:tails}
  Let $\bt \in \reals^N$.
  Using the same setup, notation, and assumptions as \cref{thm:mean_concentration},
  the odd moments of $\bt^\top \bfn_L$ are zero and the even moments larger than $2$ are super-Gaussian,
  i.e. $\Evover{\bfn_L}{ (\bt^\top \bfn_L)^{r} } \geq \Evover{ \bg \sim \normaldist{\bzero}{\bK_\text{lim}} }{ (\bt^\top \bg)^r }$
  for all even $r \geq 4$.
  Moreover, if $k_L(\cdot, \cdot)$ is bounded almost everywhere, the moment generating function $\Evover{\bfn_L}{ \exp( \bt^\top \bfn_L) }$ exists and is similarly super-Gaussian.
\end{theorem}
(See \cref{sec:proof_tails} for proofs.)
\cref{thm:mean_concentration} states that Deep GP marginals are more sharply peaked than a moment-matched Gaussian,
while \cref{thm:tails} states that they are also more heavy tailed.

\paragraph{Increasing depth leads to sharper peaks and heavier tails.}
To understand how depth affects this tail behavior, we examine the Jensen gap in \cref{eqn:mgf_jensen}.
Consider a 3-layer Deep GP $f_3(\bfn_2(\bfn_1(\cdot)))$.
If we extend \cref{eqn:mgf_jensen} to 3-layer models,
we see that the Jensen gap cascades:
\begin{equation}
  \underbracket{\Evover{\bF_1}{\Evover{\bF_2 \mid \bF_1}{ \exp \left( \frac 1 2 \bt^\top \bK_3 \bt \right) }}}_{\text{MGF of 3-layer Deep GP marginal}}
  \geq
  \underbracket{\Evover{\bF_1}{ \exp \left( \frac 1 2 \bt^\top \!\! \Evover{\bF_2 \mid \bF_1}{ \bK_3 } \bt \right)}}_{\text{MGF of 2-layer Deep GP marginal}}
  \geq
  \underbracket{\exp\left( \frac 1 2 \bt^\top \!\! \Evover{\bF_1}{\Evover{\bF_2 \mid \bF_1}{ \bK_3 }} \bt \right)}_{\text{MGF of $\normaldist{\bzero}{\Evover{\bF_1}{\Evover{\bF_2 \mid \bF_1}{ \bK_3 }}}$}},
  \nonumber
\end{equation}
where $\bK_3$ is short for $\bK_3( \bF_2(\bF_1(\bX)), \bF_2(\bF_1(\bX)) )$.
The middle term is the moment generating function of a 2-layer Deep GP marginal
(where the second layer has covariance $\Evover{\bfn_2(\cdot)}{k_3(\bfn_2(\cdot), \bfn_2(\cdot))}$).
The right-most term is the moment generating function of a (single-layer) Gaussian.
Generalizing this cascade, we see that deeper models are more heavy-tailed.
A similar analysis on the characteristic function shows that the peak at the prior mean also becomes sharper with depth (see \cref{sec:proof_tails}).

\begin{figure}[t!]
  \centering
  \includegraphics[width=\textwidth]{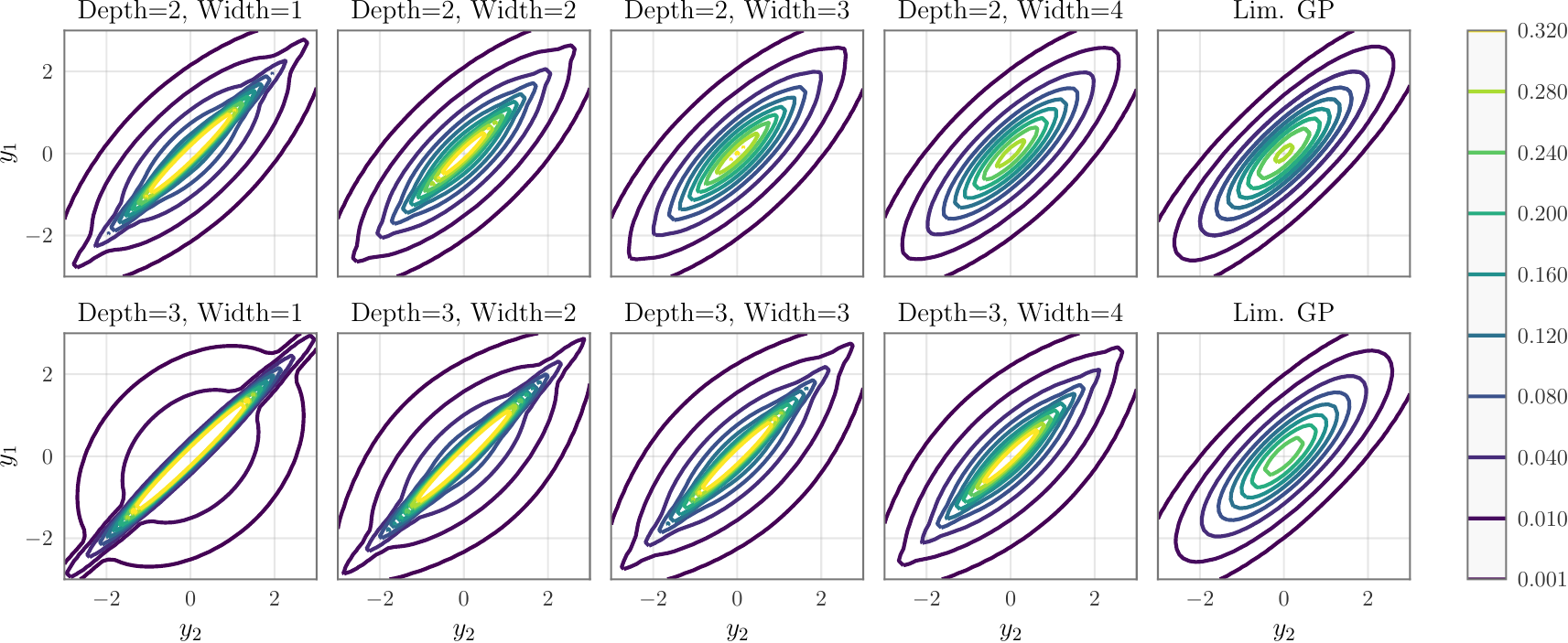}
  \caption{
    Marginal densities $p(y_1, y_2 \mid \bx_1, \bx_2)$
    for zero-mean Deep GP of various depths and widths on the $N=2$ dataset $\bx_1 = -0.5$, $\bx_2 = 0.5$.
    All 2-layer models have the same second moments (covariance is that of the 3-layer $\text{width} = 1$ RBF-RBF-RBF Deep GP).
    {\bf Left to right:} width increases, marginals become increasingly Gaussian, tails become thinner, and the peak at $[y_1, y_2] = \bzero$ loses density.
    {\bf Top to bottom:} depth increases, tails become fatter, and the peak becomes sharper.
  }
  \label{fig:density_depth_width}
\end{figure}

Adding additional layers to a Deep GP will change the model's prior covariance,
and thus the effects of depth cannot solely be explained by a tail analysis \citep{duvenaud2014avoiding,dunlop2018deep}.
Nevertheless, if we control for this change in covariance, we indeed see that depth leads to heavier tails.
In \cref{fig:density_depth_width} we compare
2-layer and 3-layer Deep GP.
The 3-layer models use GP layers with additively-decomposing RBF covariances,
while the 2-layer models use layers constructed to match the 3-layer models' prior covariance
(see \cref{sec:experimental_details} for construction details).
The $N=2$ marginal densities for the 3-layer models (bottom row) are more stretched than the 2-layer densities (top row).
We further confirm these effects in \cref{sec:additional_results}.

\paragraph{Increasing width leads to flatter peaks and Gaussian tails.}
Conversely, consider what happens when we make the model wider. 
We define the sequence of increasingly wide 2-layer Deep GP:
\begin{align}
  \left\{ \text{DGP}^{(m)}(\cdot) \triangleq \frac{1}{\sqrt{m}} \sum_{i=1}^m f^{(i)}_2(f^{(i)}_1(\cdot)) \right\},
  \qquad
  \begin{split}
    f_1^{(i)}(\cdot) \stackrel{\text{i.i.d}}{\sim} \GP{0}{k_1(\cdot, \cdot)},
    \\
    f_2^{(i)}(\cdot) \stackrel{\text{i.i.d}}{\sim} \GP{0}{k_2(\cdot, \cdot)}.
  \end{split}
  \label{eqn:additive_dgp}
\end{align}
$\text{DGP}^{(m)}(\cdot)$ is a width-$m$ Deep GP, where the second layer decomposes additively over the $m$ dimensions.
By linearity of expectation, each model in the sequence shares the same prior covariance:
$
  \mathbb{E}[\text{DGP}^{(1)}(\bx) \: \text{DGP}^{(1)}(\bx') ] =
  \mathbb{E}[\text{DGP}^{(2)}(\bx) \: \text{DGP}^{(2)}(\bx') ] =
  \ldots
  \triangleq k_\text{lim}(\bx, \bx').
$
Though each model has the same marginal covariance,
the \emph{conditional} covariance $\Evover{\bfn_2 \mid \bF_1}{ \bfn_2 \bfn_2^\top } = \frac 1 m \sum_{i=1}^m \bK_2 ( \bfn_1^{(i)}, \bfn_1^{(i)} )$
becomes increasingly concentrated around $\bK_\text{lim}(\bX, \bX)$ as $m$ increases.
This consequentially shrinks the Jensen gap in \cref{eqn:mgf_jensen},
and so the Deep GP marginals become increasingly Gaussian.
We again visualize this effect in \cref{fig:density_depth_width},
which depicts marginal densities from 2-layer and 3-layer Deep GP of various width
(see \cref{sec:experimental_details} for details).
Compared with the limiting GP (right), the width-1 densities (left) appear sharper near $[0, 0]$ and more stretched at the tails.
As width increases, the peaks and tails look increasingly Gaussian (see also \cref{fig:width_diff_depth_width} in \cref{sec:additional_results}).
In this sense, width has the opposite effect as depth---deeper marginals are less Gaussian, while wider marginals are more Gaussian.

\section{Experiments}
\subsection{Regression with Deep GP and Bayesian Neural Networks}
\label{sec:experiments_bayesian}

To isolate the effects of width and depth, each experiment compares Deep GP/Bayesian neural networks that share the same first and second prior moments,
and the Deep GP models use GP layers with universal kernels.
To remove any potential side effects from approximate inference methods,
we sample Deep GP/neural network posteriors using NUTS \citep{hoffman2014no} and do not use any stochastic inducing point \citep{salimbeni2017doubly,havasi2018inference} or finite basis \citep{cutajar2017random} approximations.
This inference is costly and scales cubically with $N$; therefore, we subsample all training datasets to $N \leq 1000$.
See \cref{sec:experimental_details} for experimental details.

\begin{figure}[t!]
  \centering
  \includegraphics[width=\linewidth]{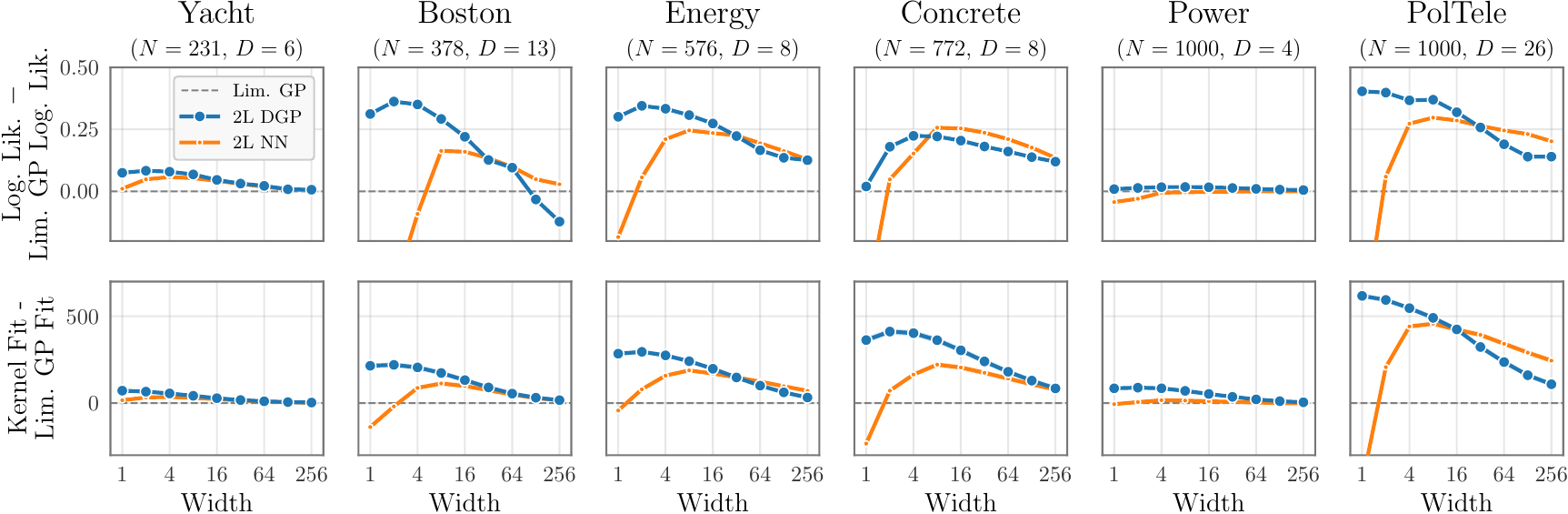}
  \caption{
    {\bf Top:} Test set log likelihood (LL) of 2-layer Deep GP (and neural networks) regression as a function of width (higher is better).
    Numbers are shifted so that $0$ corresponds to the limiting GP log likelihood.
    Narrow models achieve the best log likelihood, and performance degrades with width.
    {\bf Bottom:} Fit of the posterior kernel $k(\bfn_1(\cdot), \bfn_1(\cdot))$ on the training data, as measured by Gaussian log marginal likelihood (higher is better).
    $0$ corresponds to the limiting GP log marginal likelihood.
    Fit becomes increasingly worse with width.
  }
  \label{fig:dgp_width}
\end{figure}
\begin{table}[t!]
  \caption{%
    Test set log likelihood (LL) of Deep GP regression as a function of depth (higher is better).
    $\text{Depth}=1$ refers to the limiting GP.
    For each dataset, the models are constructed to have the same first and second moments.
    Unlike width, deeper models generally have better performance.
  }
  \label{tab:dgp_depth}
  \centering
  \resizebox{\linewidth}{!}{%
    \begin{tabular}{ccccccccc}
\toprule
\thead{Depth} & \thead{} & \thead{Yacht\\($N=231$, $D=6$)} & \thead{Boston\\($N=378$, $D=13$)} & \thead{Energy\\($N=576$, $D=8$)} & \thead{Concrete\\($N=772$, $D=8$)} & \thead{Power\\($N=1000$, $D=4$)} & \thead{PolTele\\($N=1000$, $D=26$)} \\
\midrule
1 &          &                        $-0.532$ &                          $-0.890$ &                         $-0.477$ &                           $-0.663$ &              $\mathbf{ -0.249 }$ &                            $-0.476$ \\
2 &          &                        $-0.520$ &                          $-0.684$ &                         $-0.434$ &                $\mathbf{ -0.573 }$ &                         $-0.260$ &                            $-0.381$ \\
3 &          &             $\mathbf{ -0.482 }$ &               $\mathbf{ -0.609 }$ &              $\mathbf{ -0.383 }$ &                           $-0.620$ &                         $-0.251$ &                 $\mathbf{ -0.318 }$ \\
\bottomrule
\end{tabular}

  }
\end{table}

\paragraph{Effect of width.}
We compare 2-layer Deep GP of various width on 6 regression datasets from the UCI dataset repository \citep{asuncion2007uci}
(see \cref{sec:additional_results} for 3-layer results).
The first GP layers use a RBF kernel for the prior covariance,
while the second layers use a sum of one-dimensional RBF covariance functions.
We additionally compare against the limiting (single-layer) GP with the same prior covariance ({\bf Lim. GP}).
For each dataset, we choose hyperparameters that maximize the Lim. GP log marginal likelihood.
In \cref{fig:dgp_width} (top row) we see a near-monotonic performance degradation as width increases.
The $\text{width} = 2$ optimum may represent the ``sweet spot'' for Deep GP width,
but it may instead be a side-effect of inference difficulties for $\text{width} = 1$ models
(see \cref{sec:additional_results} for a control experiment).
Regardless, as our theory predicts, \emph{width is detrimental to Deep GP predictive performance}.

We repeat the experiment for 2-layer neural networks
(and 3-layer models in \cref{sec:additional_results}),
where here the Lim. GP corresponds to the arc-cosine kernel \citep{cho2009kernel,lee2017deep}.
\cref{fig:dgp_width} indicates an optimal width with regards to test set log likelihood, usually between $8-16$ hidden units.
We expect this optimum exists (and differs from the Deep GP optimum) because narrow models have too few basis functions for these datasets.
Nevertheless, after sufficient capacity, \emph{width is harmful to Bayesian neural networks}.

\paragraph{Adaptable versus non-adaptable RKHS.}
One way to measure the ``fit'' of a kernel $k(\cdot,\cdot)$ on a regression training dataset $\bX, \by$ is the Gaussian log marginal likelihood
$\log \normaldist{ \by; \bzero }{\bK(\bX, \bX) + \sigma^2 \bI}$, where $\sigma^2$ is an observational noise parameter \citep[e.g.][]{rasmussen2006gaussian}.
To demonstrate how Deep GP/neural network posteriors correspond to adaptable RKHS mixtures,
the bottom row of \cref{fig:dgp_width} plots the ``kernel fit'' of $k_2(\bfn_1(\cdot), \bfn_1(\cdot))$ for posterior samples of $\bfn_1(\cdot)$ (see Eq.~\ref{eqn:dgp_posterior_mean}).
A higher fit corresponds to a model that is better adapted to the dataset $\bX, \by$.
We see that narrower Deep GP almost universally achieve better kernel fit than wider Deep GP,
which converge to the same fit as the limiting GP.
(Standard deviations, depicted by shaded regions, are generally imperceptible.)
Bayesian neural networks achieve best ``kernel fit'' at $8-16$ hidden units, and then converge to the limiting Deep GP with further width.

\paragraph{Effect of depth, controlling for covariance.}
\cref{tab:dgp_depth} displays Deep GP test set log likelihood as a function of depth.
Again, we isolate the tail effects of depth by ensuring that all models share the same first and second moments.
We construct a GP and a 2-layer Deep GP that match the moments of a 3-layer $\text{width} = 2$ Deep GP with RBF covariances,
and we use hyperparameters that maximize the limiting GP marginal likelihood for each dataset.
Note that computing the limiting covariance of $\geq 3$ layer models involves intractable integrals that we approximate with quadrature (see \cref{sec:limiting_kernels}).
Our findings confirm that---in this controlled setting---depth unlike width improves test set performance.

\begin{figure}[t!]
  \centering \makebox[\textwidth][c]{%
    \includegraphics[width=0.73\linewidth]{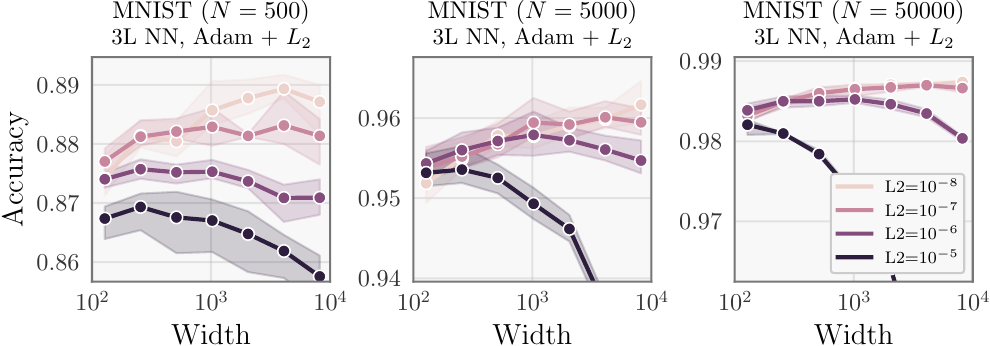}
    \hfil
    \includegraphics[width=0.2375\linewidth]{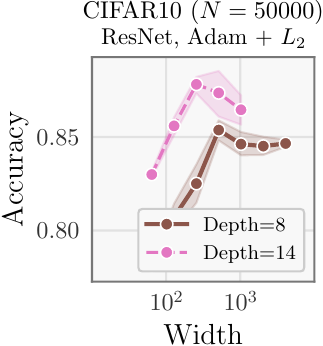}
  }
  \caption{%
    Effect of width on standard (non-Bayesian) neural networks.
    Shaded regions depict standard error.
    {\bf Left:} 3-layer MLP trained on subsets of MNIST.
    With large values of L2 regularization,
    model performance is maximized when width $\leq 1,\!000$.
    For small values of L2 regularization (e.g. $10^{-8}$, which corresponds to a prior of $\normaldist{0}{20,\!000}$ on the parameters),
    there is little accuracy loss with increasing width.
    It is possible that our theory does not apply to models with little L2 regularization which have little Bayesian interpretation.
    {\bf Right:} Wide ResNet models (8-layer and 14-layer variants) trained on CIFAR-10.
    For both depths, accuracy is optimal when width $\leq 500$.
  }
  \label{fig:mnist_cifar_acc}
\end{figure}

\subsection{Standard (Optimized, Non-Bayesian) Neural Networks}
\label{sec:experiments_optimized}
We now turn to standard (optimized, non-Bayesian) neural networks.
While our theoretical results primarily apply to full posteriors over models,
our goal is to see if our theory can also be predictive in ``real world'' neural networks without a Bayesian treatment.
There is reason to believe that our theory should be applicable in these settings, since standard neural network training with L2 regularization is equivalent to maximum a posteriori inference with Gaussian priors.
To that end, we ensure some correspondence between these experiments and our Bayesian experiments.
In particular, we measure the effects of width on networks with fixed values of L2 regularization,\footnote{
  In other words, we do not consider the regularization constant to be a hyperparameter that we optimize over
  for the purposes of these experiments.
}
which corresponds to a fixed prior on neural network parameters.
Additionally, models are trained without data augmentation, as data augmentation does not have a probabilistic interpretation \citep[]{izmailov2021bayesian}.

\cref{fig:mnist_cifar_acc} (left) depicts test set accuracy for increasingly wide models trained on MNIST \citep{lecun1998gradient}.
Each network is a MLP with 3 layers (i.e. 2 hidden layers).
Following the GP-limiting neural network construction in \cref{eqn:2layer_nn}, we scale the outputs of layer $\ell$ by $1/\sqrt{H_{\ell-1}}$.
We measure the effect of width over networks with various L2 regularization constants ($10^{-5}$, $10^{-6}$, $10^{-7}$, and $10^{-8}$)
which respectively correspond to priors of $\normaldist{0}{2}$, $\normaldist{0}{20}$, $\normaldist{0}{200}$, and $\normaldist{0}{2000}$ when $N=50,\!000$.
We train these sequences on various-sized subsets of the training data ($N=500$, $N=5,\!000$, and $N=50,\!000$).
From this figure we can observe several phenomena.
For larger values of L2 regularization, we see a distinct maximum in accuracy, typically around width $\approx 1,\!000$.
For smaller values of L2 regularization, wider models tend to perform better
(and indeed, for this dataset/model combination it appears that less regularization tends to be beneficial to overall performance).
We would note that these low regularization constants correspond to arguably unreasonable parametric priors
like $\normaldist{0}{2000}$, and so a Bayesian interpretation of these models may not be applicable.
In such settings, it is more likely that the interpolation analysis of \citet{belkin2019reconciling} is a better model of performance,
since this analysis explicitly focuses on the low-regularization setting.

\cref{fig:mnist_cifar_acc} (right)
depicts $8$- and $14$-layer ResNets \citep{he2016deep} trained on the CIFAR-10 dataset \citep{krizhevsky2009learning}.
We use the hyperparameters from the original ResNet paper,
which have been shown to be efficacious on both narrow and wide variants of ResNet models \citep{zagoruyko2016wide}.
(This training procedure uses a L2 coefficient of $10^{-4}$, which corresponds to a prior of $\normaldist{0}{0.2}$ for each parameter when $N=50,\!000$.)
For both depths, we observe that performance is optimal when width is between $500$ and $1,\!000$.
While it is possible that different hyperparameters may yield different outcomes,
these results indeed suggest that \emph{large width can adversely affect standard neural networks} once sufficient capacity is reached.

\section{Discussion}
\label{sec:discussion}

This paper shows that, across typical neural networks (with L2 regularization), Deep GP, and Bayesian neural networks,
\emph{large width can be detrimental to model performance}.

Even with these results, we can ask when width might be desirable?
First, we note that our results analyze exact posteriors or MAP solutions, and does not focus on practical considerations with regards to obtaining these solutions.
We do not consider the effect that width might have on approximate inference methods, which are commonly used with Bayesian neural networks and Deep GP in practice \citep[e.g.][]{salimbeni2017doubly,blundell2015weight,foong2020expressiveness}.
For conventional neural networks, poor conditioning and non-convexity make it challenging to obtain a MAP solution.
The optimization dynamics---which depend on numerous factors like learning rates, initializations, and choice of optimizer---%
may be improved by width, as wider models tend to have more favorable optimization landscapes \citep{nguyen2017loss,li2018visualizing,soltanolkotabi2018theoretical,arora2019fine,du2019gradient}.
Consequentially, wider models may obtain better performance due to these practical considerations.

Secondly---as noted in \cref{sec:experiments_optimized}---while we notice detrimental effects of width on neural networks with a Bayesian interpretation
(i.e. inferring a parameter posterior or optimizing parameters with L2 regularization),
we do not see these effects when such an interpretation does not exist (i.e. optimizing parameters with almost no L2 regularization).
Our theoretical findings assume that layers are conditionally Gaussian, and different priors may have different effects.
We note that much of the preliminary works on NTK assume no explicit regularization during training \citep{jacot2018neural,du2019gradient,lee2019wide}
(with the notable exception of \citet{wei2019regularization}),
and so our findings may be at odds with the empirical findings around these models \citep[e.g.][]{arora2020harnessing,garriga2019deep,geiger2020scaling}.
Moreover, recent work has proposed (non-Bayesian) infinite-width constructions
that avoid any limiting kernel behavior \citep[e.g.][]{chizat2018global,mei2018mean,golikov2020towards,yang2020feature},
and so our findings would not apply to these models.
We emphasize that our results do not conflict with these prior works, but rather reflect a different perspective.
The models we study correspond to a Gaussian prior on parameters, and so relaxing this correspondence may lessen the consequences of width that we observe.
Nevertheless, our results suggest that the inductive bias of width may be harmful, even if these undesirable effects can be avoided via careful construction.

Finally, it is worth considering when one might still choose a conventional shallow GP over a deep model.
An often-touted benefit of Gaussian processes is the ability to encode prior domain knowledge via the choice of covariance function.
In \cref{sec:dgp_second_moments}, we prove that certain prior covariances cannot be expressed by adaptable hierarchical models.
For example, a Deep GP that is composed of stationary GP layers cannot model anti-correlations a priori (\cref{thm:no_neg_covar}, \cref{sec:dgp_second_moments}),
whereas (single-layer) stationary GP can have positive and negative prior covariances.
Nevertheless, Deep GP are capable of modeling many common covariance functions, including the RBF, Mat\'ern, and rational quadratic kernels.
In \cref{sec:dgp_second_moments} we demonstrate a 2-layer Deep GP construction of any width that is capable of producing prior covariances that match  most isotropic kernels (\cref{thm:radial}, \cref{sec:dgp_second_moments}).
In other words, a Deep GP can match the first and second moments of most GP, while also offering an adaptable posterior.


\if@submission
\else
  \begin{ack}
		We would like to thank Elliott Gordon-Rodriguez for his help with the proofs.
    This work was supported by the Simons Foundation, McKnight Foundation, the Grossman Center, and the Gatsby Charitable Trust.
  \end{ack}
\fi

\addcontentsline{toc}{section}{References}
{
	\small
  \bibliographystyle{abbrvnat}
  \bibliography{citations}
}

%
%

\clearpage


\makeatletter

  \newcommand{\suptitle}{Supplementary Information for: \titl}
  \renewcommand{\@title}{\suptitle}
  \newcommand{\thanks}[1]{\footnotemark[1]}
  \renewcommand{\@author}{\authorinfo}

  \par
  \begingroup
    \renewcommand{\thefootnote}{\fnsymbol{footnote}}
    \renewcommand{\@makefnmark}{\hbox to \z@{$^{\@thefnmark}$\hss}}
    \renewcommand{\@makefntext}[1]{%
      \parindent 1em\noindent
      \hbox to 1.8em{\hss $\m@th ^{\@thefnmark}$}#1
    }
    \thispagestyle{empty}
    \@maketitle
    \@thanks
  \endgroup
  \let\maketitle\relax
  \let\thanks\relax
\makeatother

\appendix

\section{Broader Impact}
\label{sec:broader_impact}

This paper analyzes two existing classes of models: Deep GP and neural networks.
We believe that our findings will be of interest to researchers and machine learning practitioners,
offering useful guidance for Deep GP and neural network architectures.
Because we are neither introducing new algorithms nor introducing new use cases of existing algorithms, we do not foresee any major ethical impacts from this work.
However, we do note that this paper primarily focuses on how width affects performance metrics (e.g. accuracy, log likelihood, etc.)
and does not focus on other metrics that may be of interest to practitioners and society at large (e.g. interpretability, energy usage, fairness, etc.).

\section{Deep Gaussian Processes}

In this section we discuss various Deep GP facts presented throughout the main paper.

\subsection{Capacity of Deep GP}
\label{sec:dgp_capacity}

Here we formalize the claim that Deep GP have ``infinite capacity.''
Standard Gaussian processes are nonparametric, and---if the prior covariance is a universal kernel \citep{micchelli2006universal}---%
then any function (or an arbitrarily precise approximation thereof) is a draw from its prior.
A Deep GP composes multiple GP as different layers.
If all GP layers use universal kernels for covariance priors, then any (arbitrarily precise approximation of a) function $h(\cdot)$ is a draw from the Deep GP prior.
(Draw the identity function from the first $L-1$ GP layers, and then draw $h(\cdot)$ from the last GP layer.)
In this sense, Deep GP as well as standard GP can model any function to arbitrary precision and in this sense have infinite capacity.

\subsection{Connection Between Neural Networks and Deep GP}
\label{sec:dgp_degenerate}

Throughout this paper we note that feed-forward (Bayesian) neural networks are a degenerate subclass of Deep Gaussian processes.
We will now formalize this connection, which has also been noted in several previous works \citep[e.g.][]{aitchison2020bigger,ober2020global,louizos2016structured}.

To show that the neural network defined in \cref{eqn:2layer_nn} is a (degenerate) Deep GP, we must show that each of its layers corresponds to a (degenerate, vector-valued) Gaussian process.
Recall that a Gaussian process $g(\cdot) \sim \mathcal{GP}$ is a distribution over functions
where every finite marginal distribution $\bfn = [g(\bx_1), \ldots, g(\bx_N)]$ is multivariate Gaussian consistent with some covariance function.
The first layer of the \cref{eqn:2layer_nn} neural network is given by
\[
  f_1^{(i)}(\bx) = \bw_1^{(i)\top} (\bx) + \beta b_1^{(i)},
\]
Thus, the first layer corresponds to a vector-valued Gaussian process with the (degenerate) linear prior covariance.
The second layer of the \cref{eqn:2layer_nn} neural network is given by
\[
  f_2(\bz) = {\textstyle \frac{1}{\sqrt{H_1}}} \bw_2^\top \bsigma(\bz) + \beta b_2,
\]
where again the entries of $\bw_2$ and $b_2$ are i.i.d. unit Normal.
We have that $\bfn_2 \mid \bF_1 = \normaldist{\bzero}{\beta + \bsigma(\bF_1) \bsigma(\bF_1)^\top}$, where $\bF_1 = [\bfn_1(\bx_1), \ldots, \bfn_1(\bx_N)]$
and $\bsigma(\bF_1)$ corresponds to the elementwise nonlinearity $\bsigma(\cdot)$ applied to each entry of $\bF_1$.
Thus, the second layer also corresponds to a Gaussian process with a degenerate prior covariance.

\paragraph{Neural networks versus Deep GP.}
While (Bayesian) neural networks meet the definition of a Deep GP,
their covariance functions only correspond to a finite basis and therefore they do not have the same properties as nonparametric Deep GP
(i.e. the ability to model any function to arbitrary precision).
In this sense, it is common to treat neural networks and Deep GP as two separate classes of models with different predictive properties.
However, we emphasize that the theoretical results in this paper make no assumptions about whether or not a Deep GP is nonparametric,
and therefore the behaviors that we analyze are inherent to both classes of models.
In this sense, it is useful for our purposes to group nonparametric Deep GP and neural networks into a single class of models.

We also note that Deep GP and (Bayesian) neural networks can both be generalized to other hierarchical models, such as Deep Kernel Processes \citep{aitchison2021deep}.

\subsection{Factorization of Deep GP Posterior}
\label{sec:dgp_posterior_factorization}

\begin{figure}[t!]
  \centering
  \begin{tikzpicture}[node distance=1cm]
    \node[obs]                 (X)  {$\bX$} ; %
    \node[obs,    below=of X]  (x*) {$\bx^*$} ; %
    \node[latent, right=of X]  (F1)  {$\bF_1$} ; %
    \node[latent, below=of F1]  (f*1) {$\bfn^*_1$} ; %
    \node[latent, dashed, right=of F1, yshift=0.1cm]   (KFF)  {$\bK_2$} ; %
    \node[latent, dashed, right=of F1, yshift=-0.90cm]    (KF*)  {$\bk_2^*$} ; %
    \node[latent, dashed, right=of f*1, yshift=-0.1cm] (K**)  {$k_2^{**}$} ; %
    \node[latent, right=of KFF, yshift=-0.2cm] (f2)   {$\bfn_2$} ; %
    \node[latent, below=of f2]  (f*2) {$f^*_2$} ; %
    \node[obs,    right=of f2]  (y)  {$\by$} ; %

    \edge {X} {F1} ; %
    \edge {x*, F1} {f*1} ; %
    \edge {F1} {KFF} ; %
    \edge {F1, f*1} {KF*} ; %
    \edge {f*1} {K**} ; %
    \edge {KFF} {f2} ; %
    \edge {KFF, KF*, K**, f2} {f*2} ; %
    \edge {f2} {y} ; %
  \end{tikzpicture}
  \caption{
    2-layer Deep GP.
    $\bX, \by$ are the training data; $\bx^*$ is some unobserved test input.
    $\bF_1$ and $\bfn^*_1$ are the first layer outputs for the training inputs and test input, respectively.
    $\bfn_2$ and $f^*_2$ are the second layer outputs for the train/test inputs,
    which only depend on $\bF_1$, $\bfn^*_1$ through the prior covariance matrices
    $\bK_2 = \bK_2(\bF_1, \bF_1)$,
    $\bk_2^* = \bk_2(\bF_1, \bfn_1^*)$, and
    $k_2^{**} = k_2(\bfn_1^*, \bfn_1^*)$.
  }
  \label{fig:2layer_model}
\end{figure}
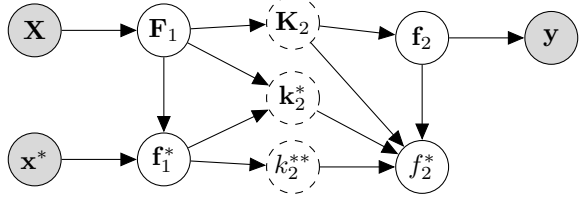

Here we supply additional details for the \cref{sec:posterior} derivation of the Deep GP posterior mean.
Consider a 2-layer zero-mean Deep-GP $f_2(\bfn_1(\cdot))$.
\cref{fig:2layer_model} depicts the relationships between these variables, using the same notation as in \cref{sec:posterior}.
Now consider the posterior distribution
\begin{equation*}
  p(f_2^*, \bfn_2, \bF_1, \bfn_1 \mid \by) \: = \:
  p(f_2^* \mid \bfn_2, \bfn_1^*, \bF_1, \by) \:
  p(\bfn_2 \mid \bF_1, \by, \bfn_1^*) \:
  p(\bfn_1^*, \bF_1 \mid \by),
\end{equation*}
where we have omitted the dependence on $\bX$ and $\bx^*$ for clarity.
Apply the rules of D-separation using \cref{fig:2layer_model},
we see that $f^*_2$ only depends on $\by$ through $\bfn_2$, and thus $f^*_2$ is conditionally independent from $\by$ given $\bfn_2$.
Furthermore, we see that $\bfn_2$ is only connected to $\bfn^*_1$ through $f^*_2$, and so $\bfn_2$ is conditionally independent from $\bfn^*_1$ if $f^*_2$ is marginalized out.
Thus, we can simplify the posterior factorization to:
\begin{equation}
  p(f_2^*, \bfn_2, \bF_1, \bfn_1 \mid \by) \: = \:
  p(f_2^* \mid \bfn_2, \bfn_1^*, \bF_1) \:
  p(\bfn_2 \mid \bF_1, \by) \:
  p(\bfn_1^*, \bF_1 \mid \by),
  \label{eqn:2layer_dgp_factorization_nokernel}
\end{equation}
Crucially, $\bfn_2$ and $f^*_2$ only depend on $\bF_1$ and $\bfn^*_1$ through $\bK_2$, $\bk_2^*$, and $k_2^{**}$:
\begin{align}
  p(\bfn_2 \mid \bK_2)
  \sim \normaldist{\bzero}{\bK_2},
  \quad
  p(f^*_2 \mid k^{**}_2, \bk^*_2, \bK_2, \bfn_2)
  \sim \normaldist{\bk_2^{*\top} \bK_2^{-1} \bfn_2}{\:\: k_2^{**} - \bk_2^{*\top} \bK_2^{-1} \bk_2^*},
  \nonumber
\end{align}
(If $f_2(\bfn_1(\cdot))$ is a neural network or any other degenerate Deep GP,
the $\bK_2^{-1}$ term can be replaced with its pseudoinverse.)
This relationship is also depicted graphically in \cref{fig:2layer_model}.
$\bK_2$, $\bk_2^*$, and $k_2^{**}$ are deterministic given $\bF_1$ and $\bfn_1^*$,
and we do not ultimately care about the values of $\bF_1$ and $\bfn_1^*$ since they are intermediate latent variables.
Therefore, we can rewrite the factorization in \cref{eqn:2layer_dgp_factorization_nokernel} where we replace $\bF_1$, $\bfn_1^*$
with $\bK_2$, $\bk_2^*$, and $k_2^{**}$:
\begin{align}
  p(f_2^*, \bfn_2, \bK_2, \bk_2^*, k_2^{**} \mid \by) \: = \:
  p(f_2^* \mid \bfn_2, \bK_2, \bk_2^*, k_2^{**}) \:
  p(\bfn_2 \mid \bK_2, \by) \:
  p(\bK_2, \bk_2^*, k_2^{**} \mid \by).
  \label{eqn:2layer_dgp_factorization_supp}
\end{align}
Applying the factorization in \cref{eqn:2layer_dgp_factorization_supp}, the posterior mean is:
\begin{align}
  \Evover{f_2^* \mid \by}{ f_2^* }
  &= \Evover{\bK_2, \bk_2^*, k_2^{**} \mid \by}{
    \Evover{ \bfn_2 \mid \bK_2, \by }{
      \expectedvalue_{ f_2^* \mid \bfn_2, \bK_2, \bk_2^*, k_2^{**} } \Bigl[ \: f_2^* \: \Bigr]
    }
  }
  \nonumber
  \\
  &= \Evover{\bK_2, \bk_2^* \mid \by}{
    \Evover{ \bfn_2 \mid \bK_2, \by }{
      \bk_2^{*\top} \bK_2^{-1} \bfn_2
    }
  }
  \label{eqn:posterior_mean_supp_step2}
  \\
  &= \expectedvalue_{ \bK_2, \bk_2^* \mid \by} \Bigl[ \:
    \bk_2^{*\top}
    \overbracket{
      \bK_2^{-1} \Evover{ \bfn_2 \mid \bK_2, \by }{ \bfn_2 }
    }^{ \balpha }
  \: \Bigr]
  \label{eqn:posterior_mean_supp_step3}
\end{align}
(Again, if $f_2(\bfn_1(\cdot))$ is a neural network,
the $\bK_2^{-1}$ term in \cref{eqn:posterior_mean_supp_step2,eqn:posterior_mean_supp_step3} can be replaced with its pseudoinverse.)
Finally, since $\bK_2, \bk^*_2$ are deterministic transforms of $\bfn_1(\bx^*)$, $\bfn_1(\bx_1)$, $\ldots$, $\bfn_1(\bx_N)$,
we can rewrite \cref{eqn:posterior_mean_supp_step3} as:
\begin{align*}
  \Evover{f_2^* \mid \by}{ f_2^* }
  &=
  \Evover{\bfn_1(\bx^*), \bfn_1(\bx_1), \ldots, \bfn_1(\bx_N) \mid \by}{ \textstyle \sum_{i=1}^N \: \alpha_i \: k_2(\bfn_1(\bx_i), \bfn_1(\bx^*))},
\end{align*}
which completes the derivation of \cref{eqn:dgp_posterior_mean} in \cref{sec:posterior}.

\subsection{A Rigorous Argument for Deep GP Posterior Collapse in the Infinite-Width Limit}
\label{sec:dgp_posterior_collapse}

In \cref{sec:posterior} (Eq.~\ref{eqn:limiting_dgp_posterior_mean}),
we argue that the posterior mean of an infinite-width Deep GP collapses to the posterior of the limiting GP.
To make this argument mathematically rigorous (and to demonstrate that this limiting posterior does not ``blow up''), we need to establish that convergence to a GP prior also implies convergence to the corresponding GP posterior.
\citet[][Proposition 1]{hron2020exact} proves that this is indeed the case, with only mild assumptions on the likelihood:
\newtheorem*{hronprop}{Proposition 1 of \citet{hron2020exact}}
\begin{hronprop}
  Assume $P_{f_n} \Rightarrow P_{f}$ (where $P_{f_n}$ in some sequence of priors, $P_{f}$ is some limiting prior, and $\Rightarrow$ denotes convergence in distribution)
  on the usual Borel product $\sigma$-algebra, Assumption~1 from \citet{hron2020exact} holds for the chosen likelihood $\ell$,
  and that $\int \ell \intd P_f > 0$. Then,
  \[
    P_{f_n \mid D} \Rightarrow P_{f \mid D}
  \]
  with $P_{f_n \mid D}$ and $P_{f \mid D}$ the Bayesian posteriors induced by the likelihood $\ell$
  and respectively the priors $P_{f_n}$ and $P_{f}$.
\end{hronprop}
Common likelihoods, such as the Gaussian likelihood for regression or the categorical likelihood for multiclass classification,
satisfy the assumptions of this proposition.

\subsection{Formalizing the Notion of ``Sufficient Capacity'' for Neural Networks}
\label{sec:dgp_sufficient_capacity}
Throughout the paper, we argue that width harms model performance once a model has ``sufficient capacity'' for a given dataset.
We intentionally keep this notion vague, since there are various measures of capacity that---%
while useful for analyzing trends in network architectures---%
are an imperfect quantification of the power of a neural network.
Nevertheless, under any standard definition of capacity, such as VC dimension, neural networks have finite capacity whereas nonparametric Deep GP with universal covariance functions have infinite capacity.
Our theory in \cref{sec:dgp,sec:posterior,sec:tails} suggests that width controls a capacity/adaptability trade-off for parametric neural networks,
analogous to other classic machine learning trade-offs.
While the optimal capacity of a neural network depends on several hard-to-measure factors and dataset-dependent features,
it stands to reason that---after sufficient width---additional neurons make the prior distribution increasingly Gaussian while offering little additional gains in modeling precision (as suggested by the orange lines in \cref{fig:dgp_width,fig:dgp3l_width}).
This is the regime that we refer to as ``sufficient capacity.''

\section{What Prior Covariance Functions can be Modeled by Deep Gaussian Processes?}
\label{sec:dgp_second_moments}
The functional properties of standard Gaussian processes are largely determined by the choice of prior covariance function \citep{rasmussen2006gaussian}.
Any positive definite function is a valid GP covariance, making it possible to encode many types of functional priors.
For Deep GP, it is reasonable to assume that its prior second moment also has significant influence on its inductive bias and functional properties.
To that end, it is of interest to determine what covariances can be modeled by Deep GP a priori.

We present two theoretical results in this section.
The first is a negative result (\cref{thm:no_neg_covar}),
which states that Deep GP with stationary GP layers can only model non-negative covariance functions a priori.
This is in contrast to standard GP, which are capable of expressing anti-correlations with stationary covariance priors.
The second is a positive result (\cref{thm:radial}),
which demonstrates a Deep GP construction capable of modeling most isotropic covariance priors.
We note that isotropic functions (e.g. RBF, Mat\'ern, rational quadratic, etc.) are some of the most common covariance priors.

\begin{theorem}
  Let $\text{\emph{DGP}}(\cdot) = f_L \circ \cdots \circ \bfn_1(\cdot)$ be a $L$-layer zero-mean Deep GP where $f_L(\cdot)$ has a mean-square continuous stationary prior covariance.
  Then $\Ev{\text{\emph{DGP}}(\bx) \: \text{\emph{DGP}}(\bx')} > 0$ for all $\bx, \bx'$.
  \label{thm:no_neg_covar}
\end{theorem}
\begin{proof}[Proof of Theorem~\ref{thm:no_neg_covar}]
  Throughout the proof, we will use the shorthand $\bfn_{\ell} = \bfn_\ell \circ \ldots \circ \bfn_1(\bx)$
  and $\bfn_{\ell}' = \bfn_\ell \circ \ldots \circ \bfn_1(\bx')$.
  Because $k_L(\cdot, \cdot)$ is stationary, we can express $\Ev{ \text{DGP}(\bx) \: \text{DGP}(\bx') }$ as:
  \begin{align*}
    \Ev{ \text{DGP}(\bx) \: \text{DGP}(\bx') }
    &= \int k_L(\bfn_{L-1} - \bfn_{L-1}') \intd p(\bfn_{L-1}, \bfn_{L-1}').
  \end{align*}
  Moreover, by Bochner's theorem, we can express $k_2(\bfn_{L-1} - \bfn_{L-1}')$ as the Fourier transform of some positive finite measure $\mu(\bxi)$:
  \begin{align}
    \Ev{ \text{DGP}(\bx) \: \text{DGP}(\bx') }
    &= \int \left( \int \exp( i \: \bxi^\top (\bfn_{L-1} - \bfn_{L-1}') )
    \intd \mu(\bxi) \right) \intd p(\bfn_{L-1}, \bfn_{L-1}').
    \label{eqn:stationary_bochner}
  \end{align}
  Note that $|\exp(i \: \cdot)|$ is bounded everywhere, and $\mu(\bxi)$ and $p(\bfn_{L-1}, \bfn_{L-1}')$ are finite measures.
  Therefore we can switch the order of integration in \cref{eqn:stationary_bochner}:
  \begin{align}
    \Ev{ \text{DGP}(\bx) \: \text{DGP}(\bx') }
    &= \int \left( \int \exp( i \: \bxi^\top (\bfn_{L-1} - \bfn_{L-1}' ) \intd p(\bfn_{L-1}, \bfn_{L-1}') \right) \intd \mu(\bxi).
    \label{eqn:stationary_bochner_switched}
    \\
    &= \int \left(
      \int \exp \left( i \:
        \begin{bmatrix} \bxi \\ -\bxi \end{bmatrix}^\top
        \begin{bmatrix} \bfn_{L-1} \\ \bfn_{L-1}' \end{bmatrix}
      \right)
      \intd p \left( \begin{bmatrix} \bfn_{L-1} \\ \bfn_{L-1}' \end{bmatrix} \right)
    \right) \intd \mu(\bxi).
    \nonumber
  \end{align}
  Applying the characteristic function lower bound for Deep GP marginals (see \cref{sec:proof_tails}, Eq.~\ref{eqn:cf_bound}):
  \begin{align}
    \Ev{ \text{DGP}(\bx) \: \text{DGP}(\bx') }
    &\geq \int \exp \left(
      - \frac 1 2
      \begin{bmatrix} \bxi \\ -\bxi \end{bmatrix}^\top
      \Ev{
        \begin{bmatrix} \bfn_{L-1} \\ \bfn_{L-1}' \end{bmatrix}
        \begin{bmatrix} \bfn_{L-1} \\ \bfn_{L-1}' \end{bmatrix}^\top
      }
      \begin{bmatrix} \bxi \\ -\bxi \end{bmatrix}
    \right) \intd \mu(\bxi).
    \label{eqn:stationary_bound}
  \end{align}
  The integrand in \cref{eqn:stationary_bound} is a real-valued exponential, and so it is strictly positive.
  Since $\mu(\bxi)$ is a positive measure, we have that $\Ev{ \text{DGP}(\bx) \: \text{DGP}(\bx') } > 0$.
\end{proof}

\begin{theorem}
  Let $k_\text{lim}(\bx, \bx') = \varphi(\Vert \bx - \bx' \Vert_2^2)$ be a mean-square continuous isotropic covariance function
  that is valid on $\reals^D \times \reals^D$ for all $D \in \mathbb N$.
  For any width $H_1 \in \mathbb N$, the exists a 2-layer Deep GP $f_2(\bfn_1(\cdot))$ with $\bfn_1(\cdot) : \reals^D \to \reals^{H_1}$
  and $f_2(\cdot) : \reals^{H_1} \to \reals$ where $\Ev{f_2(\bfn_1(\bx)) f_2(\bfn_1(\bx'))} = k_\text{lim}(\bx, \bx')$.
  \label{thm:radial}
\end{theorem}
\begin{proof}[Proof of Theorem~\ref{thm:radial}]
  A classic result from \citet[][Thm.~2]{schoenberg1938metric} is that, for any mean-square continuous isotropic covariance function $\varphi(\Vert \bx - \bx' \Vert_2^2)$
  that is valid on $\reals^D \times \reals^D$ for all $D \in \mathbb N$, there exists some positive finite measure $\mu(\beta)$
  such that
  \begin{equation}
    \varphi(\Vert \bx - \bx' \Vert_2^2) = \int \exp \left( - \frac 1 2 \Vert \bx - \bx' \Vert_2^2 \: \beta \right) \intd \mu(\beta).
  \end{equation}
  Let $\text{DGP}^{(1)}(\cdot) = f_2(f_1(\cdot))$ be a 2-layer zero-mean Deep GP with width $H_1 = 1$,
  and let $k_1(\bx, \bx') = \bx^\top \bx'$ and $k_2(z, z') = \int \exp( i \: \beta (z - z') ) \intd \mu(\beta)$.
  (By Bochner's theorem, we know that $k_2(\cdot, \cdot)$ is a valid covariance function.)
  Define $\tau = f_1(\bx) - f_1(\bx')$.
  Since $f_1(\bx)$ and $f_1(\bx')$ are jointly Gaussian:
  \[
    p \left( \begin{bmatrix} f_1(\bx) \\ f_1(\bx') \end{bmatrix} \right)
    =
    \normaldist{%
      \begin{bmatrix} 0 \\ 0 \end{bmatrix}
    }{%
      \begin{bmatrix} \bx^\top \bx & \bx^\top \bx' \\ \bx^{\prime \top} \bx & \bx^{\prime \top} \bx' \end{bmatrix}
    }
  \]
  we have that $p(\tau) = \normaldist{\bzero}{\Vert \bx - \bx' \Vert_2^2}$.
  Substituting $\tau$ into \cref{eqn:stationary_bochner_switched}, we have:
  \begin{align}
    \Ev{ \text{DGP}^{(1)}(\bx) \: \text{DGP}^{(1)}(\bx') }
    &= \int \left( \int \exp( i \: \beta \tau ) \intd p(\tau) \right) \intd \mu(\beta).
    \nonumber \\
    &= \int \exp \left( - \frac 1 2 \Vert \bx - \bx' \Vert_2^2 \: \beta \right) \intd \mu(\beta).
  \end{align}
  Thus we have a width-1 Deep GP with prior covariance $k_\text{lim}(\cdot, \cdot)$.
  We can extend this construction to 2-layer Deep GP of any width using the additive sequence defined in \cref{eqn:additive_dgp}.
\end{proof}

\section{Additional Results}
\label{sec:additional_results}

\afterpage{%
  \begin{figure}[t!]
    \centering
    \includegraphics[width=\linewidth]{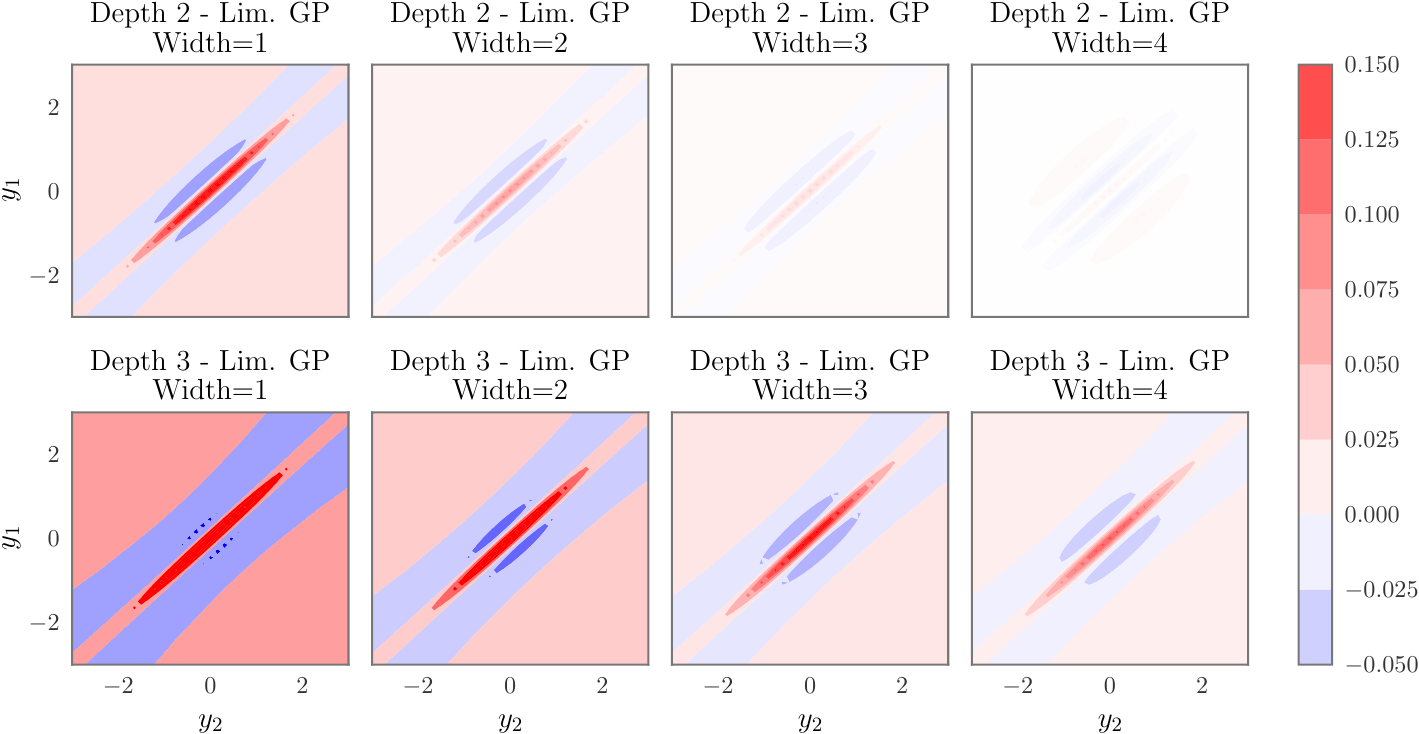}
    \caption{
      We depict the \emph{difference} between the Deep GP marginal density $p_{\text{DGP}}(y_1, y_2 \mid \bx_1, \bx_2)$
      and the limiting GP marginal density $p_{\text{Lim. GP}}(y_1, y_2 \mid \bx_1, \bx_2)$
      on the $N=2$ dataset $\bx_1 = -0.5$, $\bx_2 = 0.5$.
      Red regions correspond to values of $y_1, y_2$ where the Deep GP has more density, and vice versa for the blue regions.
      All Deep GP have heavier tails and a sharper peak than the limiting GP.
      \label{fig:gp_diff_depth_width}
    }
  \end{figure}

  \begin{figure}[t!]
    \centering
    \includegraphics[width=\linewidth]{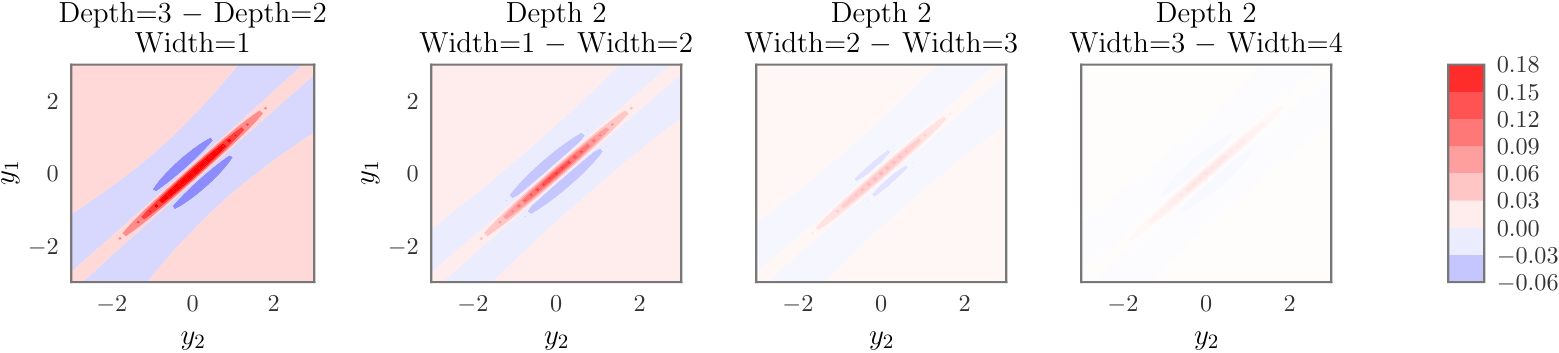}
    \caption{
      We depict the \emph{difference} between the marginal densities of various Deep GP on the $N=2$ dataset $\bx_1 = -0.5$, $\bx_2 = 0.5$.
      Red regions correspond to values of $y_1, y_2$ where the deeper/narrower Deep GP has more density, and vice versa for the blue regions.
      {\bf Left:} Comparing Deep GP of different depth.
      The 3-layer (width-1) Deep GP has a sharper peak and heavier tails than the 2-layer model, as indicated by the red regions.
      {\bf Right:} Comparing 2-layer Deep GP of different width.
      Width-$j$ models have sharper peaks and heavier tails than width-$j+1$ models, as indicated by the red regions.
      All models have the same first and second moments.
      \label{fig:width_diff_depth_width}
    }
  \end{figure}

  \begin{figure}[t!]
    \centering
    \includegraphics[width=0.6\linewidth]{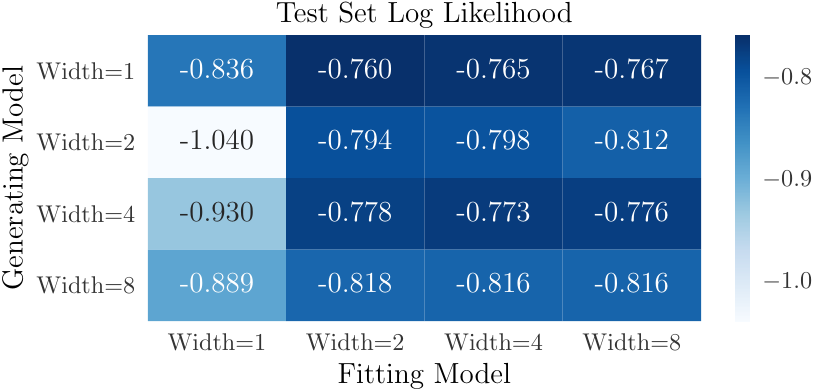}
    \caption{
      Control experiment to test how well NUTS sampling captures the Deep GP posterior.
      We generate datasets from $\text{width} \in \{ 1, 2, 4, 8 \}$ Deep GP,
      and then fit $\text{width} \in \{ 1, 2, 4, 8 \}$ Deep GP to these datasets.
      The $\text{width} \in \{ 2, 4, 8 \}$ models achieve roughly the same test set log likelihood on each dataset (higher is better).
      $\text{Width}=1$ Deep GP tend to achieve worse log likelihoods, even on a dataset generated from a $\text{width}=1$ Deep GP.
      \label{fig:control_experiment}
    }
  \end{figure}

  \clearpage
}

\subsection{Comparing Tails of Wider versus Deeper Models}
To further demonstrate that Deep GP are heavy tailed and sharply peaked, \cref{fig:gp_diff_depth_width} displays the \emph{difference} between Deep GP marginal densities and the limiting GP marginal density, i.e.:
\[
  p_\text{DGP}(y_1, y_2 \mid \bx_1, \bx_2) - p_\text{Lim. GP}(y_1, y_2 \mid \bx_1, \bx_2).
\]
Red areas correspond to values of $\by$ where the Deep GP has more density, while blue areas correspond to values where the limiting GP has more density.
Note that Deep GP of all widths and depths have red values near the $[0, 0]$ mean (corresponding to a sharper peak than the limiting GP) and red values in the upper left and lower right quadrants (corresponding to heavier tails than the limiting GP).

We also note that deeper/narrower models have heavier tails and sharper peaks than shallower/wider models (\cref{fig:width_diff_depth_width}).
The left plot shows the difference between the marginal densities of Depth-3 and Depth-2 Deep GP (with the same first and second moments).
The remaining plots show the difference between Depth-2 Deep GP of varying width (again, with the same first and second moments).

\afterpage{%
  \begin{figure}[t!]
    \centering
    \includegraphics[width=\linewidth]{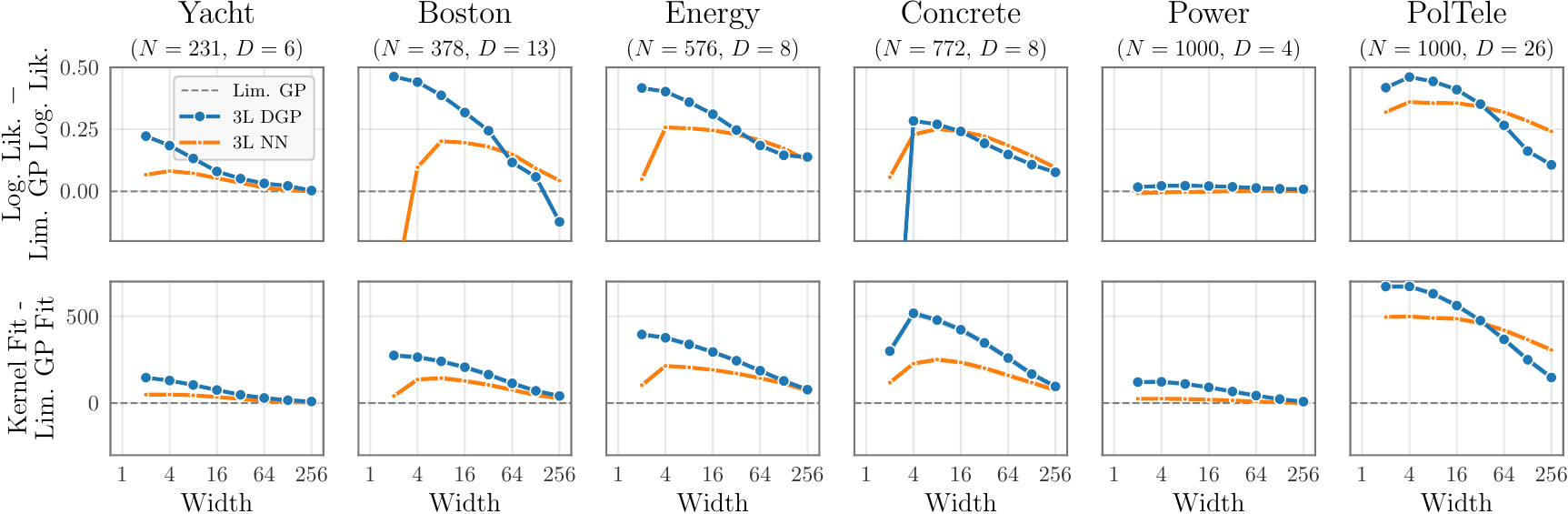}
    \caption{
      {\bf Top:} Test set log likelihood (LL) of {\bf 3-layer Deep GP} (and neural networks) as a function of width on regression datasets (higher is better).
      Numbers are shifted so that $0$ corresponds to the limiting GP log likelihood.
      {\bf Bottom:} Fit of the posterior kernel $k(\bfn_2(\bfn_1(\cdot)), \bfn_2(\bfn_1(\cdot)))$ on the training data, as measured by Gaussian log marginal likelihood (higher is better).
      $0$ corresponds to the limiting GP log marginal likelihood.
    }
    \label{fig:dgp3l_width}
  \end{figure}

  \begin{figure}[t!]
    \centering
    \includegraphics[width=\linewidth]{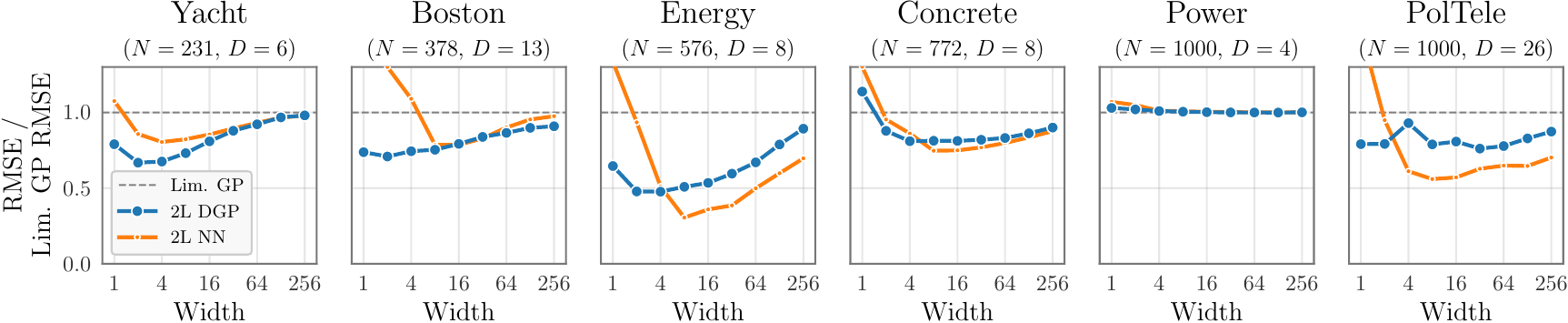}
    \caption{
      Test set root mean squared error (RMSE) of {\bf 2-layer} Deep GP (and neural networks) as a function of width on regression datasets (lower is better).
      Numbers are scaled so that $1$ corresponds to the limiting GP RMSE.
    }
    \label{fig:dgp_width_rmse}
  \end{figure}

  \begin{figure}[t!]
    \centering
    \includegraphics[width=\linewidth]{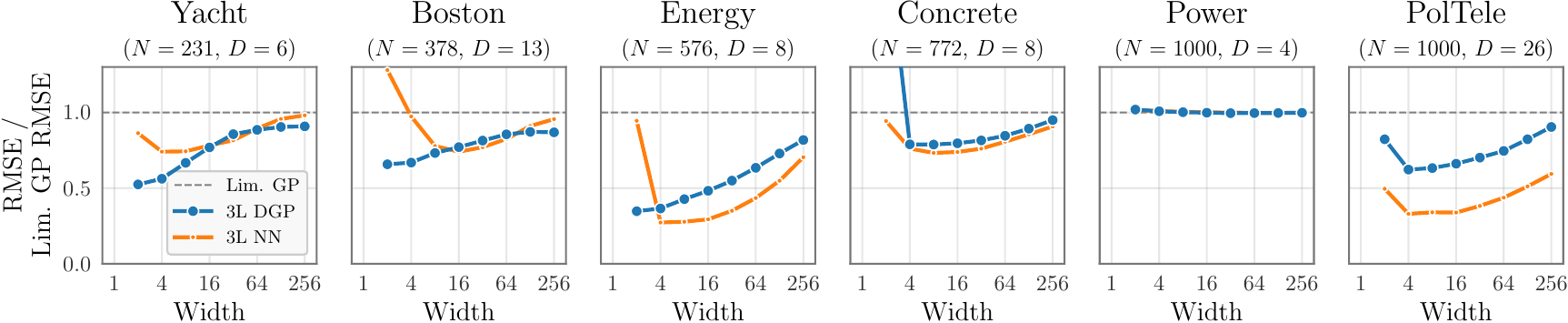}
    \caption{
      Test set root mean squared error (RMSE) of {\bf 3-layer} Deep GP (and neural networks) as a function of width on regression datasets (lower is better).
      Numbers are scaled so that $1$ corresponds to the limiting GP RMSE.
    }
    \label{fig:dgp3l_width_rmse}
  \end{figure}

  \begin{table}[t!]
    \caption{%
      Root mean squared error (RMSE) of Deep GP on UCI regression datasets as a function of depth (lower is better).
      All models for a given dataset have the same prior covariance.
    }
    \label{tab:dgp_depth_rmse}
    \centering
    \resizebox{\linewidth}{!}{%
      \begin{tabular}{ccccccccc}
\toprule
\thead{Depth} & \thead{} & \thead{Yacht\\($N=231$, $D=6$)} & \thead{Boston\\($N=378$, $D=13$)} & \thead{Energy\\($N=576$, $D=8$)} & \thead{Concrete\\($N=772$, $D=8$)} & \thead{Power\\($N=1000$, $D=4$)} & \thead{PolTele\\($N=1000$, $D=26$)} \\
\midrule
1 &          &                         $0.327$ &                           $0.643$ &                          $0.267$ &                            $0.424$ &                          $0.241$ &                             $0.305$ \\
2 &          &                         $0.240$ &                           $0.480$ &                          $0.183$ &                 $\mathbf{ 0.350 }$ &               $\mathbf{ 0.235 }$ &                             $0.229$ \\
3 &          &              $\mathbf{ 0.229 }$ &                $\mathbf{ 0.426 }$ &               $\mathbf{ 0.169 }$ &                            $0.438$ &                          $0.236$ &                  $\mathbf{ 0.218 }$ \\
\bottomrule
\end{tabular}

    }
  \end{table}

  \clearpage
}

\subsection{Control Experiment: How Well Does NUTS Sample Deep GP Posteriors?}
In order to determine if the NUTS sampler accurately captures Deep GP performance,
we perform a control experiment on synthetic data.
Specifically, we generate $N=1000$ datasets from $\text{width}=1$, $\text{width}=2$, $\text{width}=4$, and $\text{width}=8$ Deep GP,
sampling from these models at randomly-generated $\bx$ values in a $4$-dimensional input space.
Each generating Deep GP has two layers: the first uses a RBF covariance, and the second uses the sum of 1-dimensional RBF covariances.
We then train $\text{width}=1$, $\text{width}=2$, $\text{width}=4$, and $\text{width}=8$ Deep GP on each of the generated datasets,
using half the data for training and half for testing.
Our hypothesis is that $\text{width}=j$ models should at least achieve good test set performance on $\text{width}=j$ generated datasets.

\cref{fig:control_experiment} displays the test set log likelihood on the generated datasets.
We see that $\text{width} \in \{2, 4, 8\}$ models tend to achieve similar performance on each of the datasets.
On the other hand, the $\text{width}=1$ models achieve significantly worse test set performance,
even on the dataset generated by a $\text{width}=1$ model.
This suggests that the NUTS sampler is unable to converge to good posterior samples for $\text{width}=1$ models,
which may potentially explain the superior performance of $\text{width}=2$ Deep GP observed in the \cref{sec:experiments_bayesian} experiments.

\subsection{3-Layer Deep GP and 3-Layer Bayesian Neural Networks}
We extend the experiments from \cref{sec:experiments_bayesian} to 3-layer Deep GP and Bayesian neural networks.
Specifically, we measure the test set log likelihood (\cref{fig:dgp3l_width} top) and training set kernel fit (\cref{fig:dgp3l_width} bottom) on 6 regression datasets.
The Deep GP models use a standard RBF covariance function for the first layer and sums of 1-dimensional RBF covariances for the second and third layers.
The neural networks add an additional $f_3(\cdot) = {\textstyle \frac{1}{\sqrt{H_2}}} \bw_3^\top \bsigma(\cdot) + \beta b_3$
layer on top of the \cref{eqn:2layer_nn} construction, where $H_2$ is the width of the second layer and the entries of $\bw_3, b_3$ are i.i.d. unit Gaussian.

We compare models of different width, increasing the width of both hidden layers simultaneously.
It is worth noting that changing the width of the first hidden layer affects the prior second moment of the Deep GP/neural network.
Therefore, unlike the the experiments in \cref{sec:experiments_bayesian}, we are no longer ensuring that all models are moment-matched.
Nevertheless, width has the same effect for 3-layer models as it does for 2-layer models.
Width-$2$ Deep GP almost always outperform all other Deep GP, and neural networks achieve best log likelihood and kernel fit with $\leq 8$ hidden units per layer.

\subsection{Additional Figures and Tables}
\cref{fig:dgp_width_rmse,fig:dgp3l_width_rmse} report the test set root mean squared error (RMSE) of 2-layer and 3-layer models as a function of width (lower is better).
As with log likelihood, we find that width generally harms RMSE.
\cref{tab:dgp_depth_rmse} reports the test set RMSE for Deep GP of various depth (controlling the first and second prior moments, as in \cref{tab:dgp_depth}).
Again, depth is generally beneficial with regards to RMSE.

\section{Proof of Theorem~\ref{thm:dgp_gp}}
\label{sec:proof_dgp_gp}

Here we prove our main theoretical result (\cref{thm:dgp_gp}), which states that Deep GP converge to (single-layer) GP in the infinite width limit.
Following \citet{matthews2018gaussian}, we will prove that random processes with countable index sets converge in distribution to GP with countable index sets.
Thus, it is sufficient to prove that the marginals of the random process converge in distribution to multivariate Gaussians \citep[e.g.][]{billingsley2013convergence}.
While a more general treatment may be of theoretical interest, the limitation to countable index sets is sufficient for machine learning applications where we are often concerned with a finite set of events.

Before arriving at a general result, we begin with specialized proofs for two subclasses of Deep GP:
\begin{enumerate*}
  \item those that use \emph{additively-decomposing prior covariance functions} in each GP layer (\cref{lemma:additive_dgp}), and
  \item those that use \emph{isotropic prior covariance functions} (\cref{lemma:isotropic_dgp}).
\end{enumerate*}
We note that these two cases include many ``textbook'' covariance functions (i.e. the kernels described by \citet{genton2001classes} or \citet[][Ch. 4]{rasmussen2006gaussian}).
Afterwards, we present a more general result.
We note that the assumptions required for the general result are rather minimal, and are indeed satisfied by most Deep GP architectures (or arbitrarily-precise approximations thereof).

\subsection{Warmup 1: Deep GP with Additive and/or Isotropic Covariance Functions}
\label{sec:additive_dgp_proof}

The first case we will explore is Deep GP with covariance functions that decompose additively.
We will assume that the output of the additive composition is scaled to account for the input dimensionality.
This additive decomposition suggests a straightforward application of the strong law of large numbers.
\begin{observation}
  \label{lemma:additive_dgp}
  Let $f_2(\bfn_1(\bx))$ be a 2-layer zero-mean Deep GP, where $k_2(\cdot, \cdot): \reals^{H_1} \times \reals^{H_1} \to \reals$
  can be written in the form:
  \begin{equation}
    k_2(\bz, \bz') = \frac{1}{H_1} \sum_{i=1}^{H_1} k_2^{(\text{comp})} \left( z_i, z'_i \right),
    \label{eqn:additive_covar}
  \end{equation}
  where $k_2^{(\text{comp})}(\cdot, \cdot) : \reals \times \reals \to \reals$ is a positive definite function that is bounded by some polynomial:
  \[
    | k_2^{(\text{comp})}(z, z') | \leq \sum_{j=0}^R \sum_{k=0}^{j} a_{j,k} | z^{j-k} z^{\prime k} |
  \]
  for some $R < \infty$ and constants $a_{j,k} > 0$.
  Additionally, assume that $| k_1(\bx, \bx') | < \infty$ for all finite $\bx$, $\bx'$.
  Then the conditional covariance $\Ev{ f_2(\bfn_1(\bx)) f_2(\bfn_1(\bx')) \mid \bfn_1(\bx), \bfn_1(\bx')}$ becomes almost surely constant as $H_1 \to \infty$ for finite $\bx, \bx'$.
\end{observation}
\begin{proof}
  We have that
  \begin{align}
    &\phantom{=} \: \lim_{H_1 \to \infty} \Ev{f_2(\bfn_1(\bx)) f_2(\bfn_1(\bx')) \mid \bfn_1(\bx) \bfn_1(\bx')}
    \nonumber
    \\
    &=
    \lim_{H_1 \to \infty} k_2(\bfn_1(\bx), \bfn_1(\bx'))
    =
    \lim_{H_1 \to \infty} \frac{1}{H_1}
    \sum_{i=1}^{H_1} k_2^{(\text{comp})}(f_1^{(i)}(\bx), f_1^{(i)}(\bx')).
    \label{eqn:additive_limit}
  \end{align}
  Note that all the $k_2^{(\text{comp})}(f^{(i)}(\bx), f^{(i)}(\bx'))$ terms are i.i.d. by construction.
  Moreover, since $k_1(\bx, \bx')$ is finite for all finite $\bx, \bx'$,
  all moments of $f_1^{(i)}(\bx), f_1^{(i)}(\bx')$ will also be finite.
  Using our assumptions on $k_2^{(\text{comp})}(\cdot, \cdot)$, we have:
  \begin{align*}
    \left\vert \Evover{f_1^{(i)}(\cdot)}{ k_2^{(\text{comp})}(f_1^{(i)}(\bx), f_1^{(i)}(\bx'))} \right\vert
    &\leq
    \Evover{f_1^{(i)}(\cdot)}{ \left\vert k_2^{(\text{comp})}(f_1^{(i)}(\bx), f_1^{(i)}(\bx')) \right\vert}
    \\
    &\leq
    \sum_{j=0}^R \sum_{k=0}^{j} a_{j,k}
    \Evover{f_1^{(i)}(\cdot)}{ \left\vert f_1^{(i)}(\bx)^{j-k} f_1^{(i)}(\bx')^k \right\vert }
    \\
    &\leq
    \sum_{j=0}^R \sum_{k=0}^{j} a_{j,k}
    \sqrt{\Evover{f_1^{(i)}(\cdot)}{ \left\vert f_1^{(i)}(\bx)^{2j-2k} f_1^{(i)}(\bx')^{2k} \right\vert }}
    \\
    &< \infty.
    \tag{moments of $f_1^{(i)}(\bx), f_1^{(i)}(\bx')$ are finite}
  \end{align*}
  Therefore, we can apply the strong law of large numbers to the limit in \cref{eqn:additive_limit}
  which thus is almost surely constant.
\end{proof}

Neural networks are one common case of a Deep GP with additive covariance functions.
Combining \cref{lemma:additive_dgp} with \cref{lemma:nn_cf} gives the classic GP convergence result first discovered by \citet{neal1995bayesian}.
We note however that, when additive structure exists, we do not need to rely on the covariance perspective of \cref{lemma:nn_cf,lemma:dgp_cf}, as we can prove GP convergence using the central limit analysis of \citep{neal1995bayesian,lee2017deep,matthews2018gaussian}.

\subsection{Warmup 2: Deep GP with Continuous Isotropic Covariance Functions}
\label{sec:isotropic_dgp_proof}

The most common Deep GP architectures use RBF or Mat\'ern covariance functions for the GP layers \citep[e.g.][]{damianou2013deep,damianou2015deep,bui2016deep,salimbeni2017doubly,havasi2018inference}.
These covariance functions belong to the class of \emph{isotropic kernels}, which are covariances that can be written as a function of Euclidean distance:
\[
  k(\bz, \bz') = \varphi( \Vert \bz - \bz' \Vert_2^2 ).
\]
Similar to \cref{lemma:additive_dgp}, our analysis of isotropic Deep GP relies on the strong law of large numbers.
Again, we will assume that the covariance functions are scaled to account for the input dimensionality.

\begin{observation}
  \label{lemma:isotropic_dgp}
  Define $\mathbb{A} \triangleq \cup_{d=1}^\infty ( \reals^d \times \reals^d)$.
  Let $f_2(\bfn_1(\bx))$ be a 2-layer zero-mean Deep GP, where $k_2(\cdot, \cdot): \mathbb A \to \reals$
  is a continuous isotropic kernel that can be written in the form:
  \begin{equation}
    k_2(\bz, \bz') = \begin{cases}
      \varphi \left( (\bz - \bz')^2 \right)  & \bz, \bz' \in \reals \\
      \varphi \left( \frac{1}{2} \Vert \bz - \bz' \Vert_2^2 \right)  & \bz, \bz' \in \reals^2 \\
      \varphi \left( \frac{1}{3} \Vert \bz - \bz' \Vert_2^2 \right)  & \bz, \bz' \in \reals^3 \\
      \varphi \left( \frac{1}{4} \Vert \bz - \bz' \Vert_2^2 \right)  & \bz, \bz' \in \reals^4 \\
      \vdots
    \end{cases}.
    \label{eqn:isotropic_covar}
  \end{equation}
  Additionally, assume that $| k_1(\bx, \bx') | < \infty$ for all finite $\bx$, $\bx'$.
  Then the conditional covariance $\Ev{ f_2(\bfn_1(\bx)) f_2(\bfn_1(\bx')) \mid \bfn_1(\bx), \bfn_1(\bx')}$ becomes almost surely constant as $H_1 \to \infty$ for finite $\bx, \bx'$.
\end{observation}
\begin{proof}
  We have that
  \begin{align}
    &\phantom{=} \: \lim_{H_1 \to \infty} \Ev{f_2(\bfn_1(\bx)) f_2(\bfn_1(\bx')) \mid \bfn_1(\bx) \bfn_1(\bx')}
    \nonumber
    \\
    &=
    \lim_{H_1 \to \infty} k_2(\bfn_1(\bx), \bfn_1(\bx'))
    =
    \lim_{H_1 \to \infty} \varphi \left( \frac{1}{H_1} \Vert \bfn_1(\bx) - \bfn_1(\bx') \Vert_2^2 \right)
    \nonumber
    \\
    &=
    \varphi \left( \lim_{H_1 \to \infty} \frac 1 {H_1} \sum_{i=1}^{H_1} \left(
      f_1^{(i)}(\bx) - f_1^{(i)}(\bx')
    \right)^2 \right)
    \label{eqn:isotropic_limit}
  \end{align}
  where we can move the limit inside $\varphi(\cdot)$ by the continuity assumption.
  Note that $f_1^{(i)}(\bx) - f_1^{(i)}(\bx') \stackrel{\text{i.i.d}}{\sim} \normaldist{ 0 }{ \sigma^2 }$,
  where $\sigma^2 = k_1(\bx, \bx) + k_1(\bx', \bx') - 2 k_1(\bx, \bx')$.
  Since $k_1(\bx, \bx')$ is finite---%
  and thus $\sigma^2 = \Ev{(f_1^{(i)}(\bx) - f_1^{(i)}(\bx'))^2}$ is finite---%
  \cref{eqn:isotropic_limit} is almost surely constant by the strong law of large numbers.
\end{proof}
Combining \cref{lemma:isotropic_dgp} with \cref{lemma:dgp_cf}, we have that Deep GP with isotropic covariance functions
become GP in their infinite width limit.

\subsection{Assumptions Required for a General Result}
\label{sec:assumptions}

We note that the most common Deep GP/neural network architectures either decompose additively or use isotropic covariance functions \citep{damianou2013deep,bui2016deep,salimbeni2017doubly,agrawal2020wide,aitchison2020bigger}.
In these cases, \cref{lemma:additive_dgp,lemma:isotropic_dgp} are sufficient to prove GP convergence.
To prove a more general result, we will need to make an additional set of assumptions about the GP covariance functions.
We emphasize that these assumptions are sufficient but not necessary conditions,
yet they hold for most Deep GP architectures (or arbitrarily-accurate approximations thereof).
Informally, these assumptions are:
\begin{enumerate}
  \item each covariance function \emph{``scales reasonably''} with the dimensionality of its inputs;
  \item each covariance function has a \emph{Lesbegue-Stieltjes representation}; and
  \item each covariance function is \emph{bounded}.
\end{enumerate}
These assumptions are---in practice---very minimal.
The first item ensures that the prior covariance doesn't ``blow up'' for high-dimensional data (e.g. the covariance does not converge to a constant for all inputs).
The second item admits all but the most pathological covariance functions \citep{yaglom1987correlation}.
The last item may at first seem unreasonable, since some common covariance functions are unbounded (e.g. linear kernel, polynomial kernel, etc.).
However, we note that many covariance functions are bounded on any compact domain, and therefore from a practical perspective we can approximate any of these unbounded covariances to any arbitrary precision.

\begin{assumption}[Covariance functions have Lesbegue-Stieltjes representations with compact spectral support]
  \label{ass:kernels_harmonizable}
  \normalfont
    We assume that all covariance functions of interest $k(\bz, \bz'): \reals^D \times \reals^D \to \reals$ can be represented by a Lesbegue-Stieltjes integral of the following form:
    \begin{equation}
      k(\bz, \bz') = \int \rho \left( \frac{1}{D}
      \sum_{i=1}^D \phi \left( z_i \xi_i - z'_i \xi'_i \right) \right)
      \intd \mu((\xi_1, \xi'_1), \ldots, (\xi_D, \xi'_D)),
      \label{eqn:lesbegue_steiltjes}
    \end{equation}
    where
    \begin{itemize}
      \item $\mu((\xi_1, \xi'_1), \ldots, (\xi_D, \xi_D'))$ is a positive definite function of bounded variation and compact support: i.e.:
        there exists some constant $C < \infty$ such that $\mu( (\reals \times \reals)^D - ( [-C, C] \times [-C, C])^D )  = 0$;
      \item $\phi : \reals \to \reals$ is bounded above and below by a finite polynomial; and
      \item $\rho(z) : \reals \to \reals$ is a bounded and continuous function.
    \end{itemize}
    We will show that this formulation---though complex---admits additive kernels, isotropic kernels, and most other non-pathological kernels (or arbitrarily-precise approximations thereof).
\end{assumption}

\begin{assumption}[Covariance functions scale with dimensionality]
  \label{ass:kernels_dimensionality}
  \normalfont
    Let $\mathbb{A} = \cup_{d=1}^\infty ( \reals^d \times \reals^d)$.
    We assume that the covariance function $k(\cdot, \cdot): \mathbb{A} \to \reals$ scales \emph{reasonably with dimensionality};
    that is, it can be written in the following form:
    \[
      k(\bz, \bz') = \begin{cases}
        k^{(1)}(\bz, \bz'), & \bz, \bz' \in \reals, \\
        k^{(2)}(\bz, \bz'), & \bz, \bz' \in \reals^2, \\
        k^{(3)}(\bz, \bz'), & \bz, \bz' \in \reals^3, \\
        \vdots
      \end{cases}
    \]
    where the $k^{(j)}(\bz, \bz')$ satisfy \cref{ass:kernels_harmonizable}:
    \begin{equation}
      \begin{gathered}
        k^{(j)}(\bz, \bz') = \int \rho \left( \frac{1}{j}
        \sum_{i=1}^j \phi \left( z_i \xi_i - z'_i \xi'_i \right) \right)
        \intd \mu_j((\xi_1, \xi'_1), \ldots, (\xi_j, \xi_j')),
      \end{gathered}
      \nonumber
    \end{equation}
    and the measures $\mu_j$ satisfy the Kolmogorov consistency condition:
    \begin{equation}
      \begin{gathered}
        \mu_{j + k} \left( \mathcal E \times \reals^k \right) = \mu_j \left( \mathcal E \right)
        \text{ for every $j$, $k \geq 1$, and every Borel set $\mathcal E \subset \reals^j$}.
      \end{gathered}
      \nonumber
    \end{equation}
\end{assumption}

\subsection{An Exhaustive Discussion About Assumptions~\ref{ass:kernels_harmonizable} and \ref{ass:kernels_dimensionality}}
We again reiterate that \cref{ass:kernels_harmonizable,ass:kernels_dimensionality} are not necessary conditions.
If more is known about the Deep GP architecture (e.g. all covariance functions are additive or isotropic, etc.), then it is possible to use a smaller and simpler set of assumptions.
We nevertheless argue that---in practice, \cref{ass:kernels_harmonizable,ass:kernels_dimensionality} are indeed very general.

\paragraph{All but the most pathological covariance functions can be expressed (or well-approximated) by \cref{ass:kernels_harmonizable}.}
If we choose $\rho(\cdot) = \cos(\cdot)$, $\phi(\cdot)$ to be the identity function, and $\mu$ to be any measure of bounded variation, then \cref{eqn:lesbegue_steiltjes} reduces to:
\begin{equation}
  k(\bz, \bz') = \int \cos \left( \frac{1}{D}
  \sum_{i=1}^D \left( \bz^\top \bxi - \bz^{\prime\top} \bxi' \right) \right)
  \intd \mu(\bxi, \bxi'),
  \label{eqn:fourier_steiltjes}
\end{equation}
where $\bxi$ and $\bxi'$ are equal to $[\xi_1, \ldots, \xi_D]$ and $[\xi'_1, \ldots, \xi'_D]$, respectively.
This is a Fourier-Stieljes integral, and almost all (bounded) covariance functions encountered in the machine learning literature can be written in this form \citep{genton2001classes}.
Such covariance functions are known as \emph{harmonizable kernels}.
When the measure $\mu$ is concentrated on the diagonal, then \cref{eqn:fourier_steiltjes} reduces to Bochner's representation of stationary covariance functions \citep[][Ch. 4]{rasmussen2006gaussian}.
More generally, a non-diagonal $\mu$ results in non-stationary covariance functions,
and \citet[][Sec. 26.4]{yaglom1987correlation} argues that (bounded) covariances that \emph{cannot} be expressed by \cref{eqn:fourier_steiltjes} tend to be pathological in nature.

Of course, any covariance function that can be expressed as \cref{eqn:fourier_steiltjes} is necessarily \emph{bounded},
since $\cos(\cdot)$ is bounded and the measure $\mu$ has bounded variation by assumption.
The most common unbounded covariance functions are dot-product kernels, such as the linear kernel ($\beta^2 + \bx^\top \bx'$ for some constant $\beta > 0$)
or the ``ReLU kernel'' ($\beta^2 + \bsigma(\bx)^\top \bsigma(\bx')$, where $\bsigma(\cdot) = \max\{ \bzero, \cdot \}$).
Importantly, these covariance functions meet the additive structure condition of \cref{lemma:additive_dgp}, and so they do not require the general treatment of \cref{ass:kernels_harmonizable,ass:kernels_dimensionality}.
However, we would also note that these covariances are bounded on any compact domain, so we can approximate them to arbitrary precision by replacing $\bx$ with $({\bx}/{\Vert \bx \Vert_\infty}) \min\{ \Vert \bx \Vert_\infty, B\}$ for any $B < \infty$ and similarly for $\bx'$.

The other simplifying assumption is that the spectral measure $\mu$ has compact support.
Again, even if $k(\cdot, \cdot)$ corresponds to a spectral measure with infinite support, it can be approximated to arbitrary precision by replacing $\bxi$ with $({\bxi }/{\Vert \bxi \Vert_\infty}) \min \{ \Vert \bxi \Vert_\infty, B \}$ for some constant $B$.

\paragraph{\cref{ass:kernels_dimensionality} captures natural ways of scaling to dimensionality.}
The two requirements of \cref{ass:kernels_dimensionality} are mechanisms for defining reasonable sequences of covariance functions.
The consistency requirement on $\mu_j$ ensures that covariance functions do not ``change significantly'' as dimensionality increases.
The $1/j$ term simply prevents covariances from becoming unbounded or degenerate as $j \to \infty$.
We note that---by choosing appropriate $\rho(\cdot)$ and $\phi(\cdot)$ functions in \cref{eqn:lesbegue_steiltjes}---this $1/j$ term can correspond to ``natural'' scaling rates.
For example:
\begin{itemize}
  \item If $k(\cdot, \cdot)$ is isotropic, then \cref{eqn:lesbegue_steiltjes} can be reduced to:
    \[
      k_j(\bz, \bz') = \varphi \left( \frac 1 j \sum_{i=1}^j \left( z_i - z_i' \right)^2 \right)
      = \varphi \left( \frac 1 j \left\Vert \bz_i - \bz_i' \right\Vert^2 \right),
    \]
    where $\varphi(\cdot)$ is a continuous positive definite function.
    We get this by setting $\rho(\cdot) = \varphi(\cdot)$, $\phi(\cdot) = (\cdot)^2$, and by setting $\mu_j$ to be atomic.
    This is the scaling of isotropic covariances studied in \cref{sec:isotropic_dgp_proof}.

  \item If $k(\cdot, \cdot)$ is additive, then \cref{eqn:lesbegue_steiltjes} can take the form:
    \begin{equation}
      k_j(\bz, \bz') = \frac 1 j \sum_{i=1}^j \int \cos \left( z_i \xi_i - z_i' \xi_i' \right) \intd \mu(\xi_i, \xi_i'),
      \label{eqn:additive_lesbegue}
    \end{equation}
    where $\mu$ is of bounded variation.
    We get this by setting $\rho(z) = (z/|z|) \min\{ |z|, B \}$ for some sufficiently large constant $B$ and by setting $\phi(\cdot) = \cos(\cdot)$.\footnote{
      Since the integral $\int \cos \left( z_i \xi_i - z_i' \xi_i' \right) \intd \mu(\xi_i, \xi_i')$ is bounded, there exists some constant $B$
      such that $\rho(z) = z/|z| \min\{ |z|, B \}$ is effectively equal to the identity function in \cref{eqn:additive_lesbegue}.
    }
    This is now the sum of 1D harmonizable covariance functions using the scaling studied in \cref{sec:additive_dgp_proof}.
\end{itemize}

\subsection{A General Result}
With \cref{ass:kernels_harmonizable,ass:kernels_dimensionality},
we show that the conditional covariance of a 2-layer zero-mean Deep GP becomes almost surely constant in the limit of infinite-width.

\begin{lemma}\label{lemma:dgp_cond_covar}
  Let $f_2(\bfn_1(\bx))$ be a 2-layer zero-mean Deep GP, where $k_2(\cdot, \cdot)$ satisfies \cref{ass:kernels_harmonizable,ass:kernels_dimensionality}.
  The conditional covariance $\Ev{ f_2(\bfn_1(\bx)) f_2(\bfn_1(\bx')) \mid \bfn_1(\bx), \bfn_1(\bx')}$ becomes almost surely constant as $H_1 \to \infty$.
\end{lemma}
As with our warmups (\cref{lemma:additive_dgp,lemma:isotropic_dgp}),
the proof of \cref{lemma:dgp_cond_covar} essentially boils down to applying the strong law of large numbers.
The primary complication of this proof is ensuring that \cref{ass:kernels_harmonizable,ass:kernels_dimensionality} satisfy the conditions necessary to invoke the strong law.
\begin{proof}
  By \cref{ass:kernels_harmonizable,ass:kernels_dimensionality}, the limiting conditional covariance can be written as:
  \begin{align}
    &\phantom{=}\:
    \lim_{H_1 \to \infty} \Ev{ f_2(\bfn_1(\bx)) f_2(\bfn_1(\bx')) \mid \bfn_1(\bx), \bfn_1(\bx')}
    \nonumber
    \\
    &=
    \lim_{H_{1} \to \infty} k_2(\bfn_1(\bx), \bfn_1(\bx'))
    \nonumber
    \\
    &= \lim_{H_{1} \to \infty} \int \rho \left( \frac{1}{H_{1}} \sum_{i=1}^{H_1} \phi \left(  f_1^{(i)}(\bx) \xi_i - f_1^{(i)}(\bx') \xi_i' \right) \right)
      \intd \mu_{(H_{1})}((\xi_1, \xi'_1), \ldots, (\xi_{H_1}, \xi_{H_1}')).
    \label{eqn:full_lim_covar_finite_measure}
  \end{align}
  By the consistency requirement in \cref{ass:kernels_dimensionality}, we know that there exists a unique probability measure $\mu_\infty$ on the Borel product sigma-algebra over $\reals^\infty$
  such that, for any $H_1 \in \mathbb N$ and any Borel subset $\mathcal E \in \reals^{H_1}$, we have $\mu_{H_1} \left( \mathcal E \right) = \mu_\infty \left( \mathcal E \times \reals \times \reals \times \ldots \right)$.
  Therefore, we can rewrite \cref{eqn:full_lim_covar_finite_measure} as
  \begin{align}
    \lim_{H_{1} \to \infty} \int
    g^{(H_1)}_{\bfn_1(\bx), \bfn_1(\bx')} \left( \bxi, \bxi' \right)
    \intd \mu_{\infty}(\bxi, \bxi'),
    \label{eqn:full_lim_covar}
  \end{align}
  where $\bfn_1(\bx)$, $\bfn_1(\bx')$, $\bxi$, and $\bxi'$ are infinite dimensional vectors ($\bfn_1(\bx) = [f_1^{(1)}(\bx), f^{(2)}_1(\bx), \ldots]$, $\bxi = [\xi_1, \xi_2, \ldots]$),
  and $g^{(H_1)}_{\bfn_1(\bx), \bfn_1(\bx')}$ is the random function given by
  \begin{equation}
    g^{(H_1)}_{\bfn_1(\bx), \bfn_1(\bx')} \left( \bxi, \bxi' \right)
    = \rho \left( \frac{1}{H_{1}} \sum_{i=1}^{H_1} \phi \left(  f_1^{(i)}(\bx) \xi_i - f^{(i)}(\bx') \xi_i' \right) \right).
    \label{eqn:function_sequence}
  \end{equation}
  Note that $ g^{(H_1)}_{\bfn_1(\bx), \bfn_1(\bx')} \left( \bxi, \bxi' \right)$ is bounded (because $\rho(\cdot)$ is bounded).
  Moreover, since each of the $\mu_j$ have compact support, $\mu_\infty$ will have compact support as well.
  Therefore, we can consider the domain of $g^{(H_1)}_{\bfn_1(\bx), \bfn_1(\bx')} \left( \bxi, \bxi' \right)$
  to be $[-C, C]^D \times [-C, C]^D$ for some constant $C < \infty$,
  and the range to be $[-B, B]$ for some constant $B < \infty$.
  Now consider the limit of \cref{eqn:function_sequence}:
  \begin{equation}
    \lim_{H_1 \to \infty} g^{(H_1)}_{\bfn_1(\bx), \bfn_1(\bx')} \left( \bxi, \bxi' \right)
    = \rho \left( \lim_{H_1 \to \infty} \frac{1}{H_{1}} \sum_{i=1}^{H_1} \phi \left(  f_1^{(i)}(\bx) \xi_i - f_1^{(i)}(\bx') \xi_i' \right) \right),
    \label{eqn:lim_function_sequence}
  \end{equation}
  where we can bring the limit inside $\rho$ by continuity.
  Consider fixed inputs $\bxi, \bxi' \in [-C, C]^\infty$.
  The $f^{(i)}(\cdot)$ terms are i.i.d. zero-mean Gaussian by construction, and have finite variance because $k_1(\cdot, \cdot)$ is finite almost everywhere (\cref{ass:kernels_harmonizable}).
  Additionally, the $\xi_i$ and $\xi_i'$ terms lie in the finite interval $[-C, C]$.
  Consequentially, $ f_1^{(i)}(\bx) \xi_i - f_1^{(i)}(\bx') \xi'_i $ are Gaussian random variables with zero mean and bounded variance.
  Moreover, $\phi(\cdot)$ is bounded above and below by a finite polynomial (\cref{ass:kernels_harmonizable}),
  and thus $\Var{ \phi ( f_1^{(i)}(\bx) \xi_i - f_1^{(i)}(\bx') \xi'_i ) }$ is bounded by a finite linear combination of the moments of $f_1^{(i)}(\bx) \xi_i - f_1^{(i)}(\bx') \xi'_i$.
  Since the moments of a Gaussian are positive polynomial functions of the variance,
  we have that
  $\Var{ \phi ( f_1^{(i)}(\bx) \xi_i - f_1^{(i)}(\bx') \xi'_i ) }$
  is bounded, and thus
  \[
    \begin{gathered}
      \left\vert \Ev{ \phi \left( f_1^{(i)}(\bx) \xi_i - f_1^{(i)}(\bx') \xi'_i \right) } \right\vert < \infty, \qquad
      \sum_{i=1}^\infty \frac 1 {i^2} \Var{ \phi \left( f_1^{(i)}(\bx) \xi_i - f_1^{(i)}(\bx') \xi'_i \right) } < \infty.
    \end{gathered}
  \]
  We therefore satisfy the strong law of large numbers conditions, and so
  \cref{eqn:lim_function_sequence}
  converges to a constant almost surely for any $\bxi, \bxi' \in [-C, C]^\infty$.
  In other words,
  \begin{equation}
    \lim_{H_1 \to \infty}
    g^{(H_1)}_{\bfn_1(\bx), \bfn_1(\bx')} \left( \bxi, \bxi' \right)
    =
    \lim_{H_1 \to \infty} \frac{1}{H_{1}} \sum_{i=1}^{H_1} \phi \left(  f_1^{(i)}(\bx) \xi_i - f_i^{(i)}(\bx') \xi_i' \right)
    =
    \text{const. a.s.}
    \label{eqn:inner_limit_converges}
  \end{equation}
  Since this holds for all $\bxi, \bxi' \in [-C, C]^\infty$, we have:
  \begin{equation}
    \int \lim_{H_1 \to \infty}
    g^{(H_1)}_{\bfn_1(\bx), \bfn_1(\bx')} \left( \bxi, \bxi' \right)
    \intd \mu_{\infty}(\bxi, \bxi') =
    \text{const. a.s.}
    \label{eqn:int_inner_limit_converges}
  \end{equation}

  To finish off, we must show that dominated convergence implies that
  \begin{align}
    \lim_{H_1 \to \infty} \int
    g^{(H_1)}_{\bfn_1(\bx), \bfn_1(\bx')} \left( \bxi, \bxi' \right)
    \intd \mu_{\infty}(\bxi, \bxi')
    =
    \int \lim_{H_1 \to \infty}
    g^{(H_1)}_{\bfn_1(\bx), \bfn_1(\bx')} \left( \bxi, \bxi' \right)
    \intd \mu_{\infty}(\bxi, \bxi').
    \label{eqn:dct}
  \end{align}
  For all $H_1 \in \mathbb N$, we have that
  $\left| g^{(H_1)}_{\bfn_1(\bx), \bfn_1(\bx')} \right| \leq B$
  is bounded and therefore trivially dominated by $\int B \intd \mu_{\infty}(\bxi, \bxi') < \infty$.
  Now consider any fixed value of $\bfn_1(\bx), \bfn_1(\bx')$.
  By \cref{eqn:inner_limit_converges}, we have that
  $g^{(H_1)}_{\bfn_1(\bx), \bfn_1(\bx')} \left( \bxi, \bxi' \right)$
  will converge pointwise with respect to $\bxi, \bxi'$
  except when $\bfn_1(\bx), \bfn_1(\bx')$ comes from a set of measure $0$.
  Therefore, \cref{eqn:dct} holds with probability $1$ (where the probablity is taken with respect to $\bfn_1(\bx)$, $\bfn_1(\bx')$).
  Combining \cref{eqn:full_lim_covar,eqn:dct,eqn:int_inner_limit_converges} completes the proof.
\end{proof}

The proof of \cref{thm:dgp_gp} follows from applying \cref{lemma:dgp_cf,lemma:dgp_cond_covar}.
\newtheorem*{dgpthm}{Theorem~\ref{thm:dgp_gp} (Restated)}
\begin{dgpthm}
  \vspace{0.5em}
  Let $f_{L} \circ \ldots \circ \bfn_{1} \left( \bx \right)$ be a zero-mean Deep GP (Eq.~\ref{eqn:dgp_def}),
  where each layer satisfies \cref{ass:kernels_harmonizable,ass:kernels_dimensionality}.
  Then $\lim_{H_{L-1} \to \infty} \cdots \lim_{H_1 \to \infty} f_{L} \circ \ldots \circ \bfn_{1} \left( \bx \right)$
  converges in distribution to a (single-layer) GP.
\end{dgpthm}
\begin{proof}
In the two layer case, combining \cref{lemma:dgp_cf,lemma:dgp_cond_covar} gives us:
\begin{equation}
  \lim_{H_1 \to \infty} \Ev{ \exp \left( i \bt^\top \bfn_2 \right) }
  =
  \exp \left(
    -\frac 1 2 \bt^\top \Ev{ \bfn_2 \bfn_2^\top } \bt
  \right)
  \quad
  \text{for all } \bt \in \reals^N.
\end{equation}
Note that this is the characteristic function of a zero-mean multivariate Gaussian with covariance $\Ev{ \bfn_2 \bfn_2^\top }.$
Thus by L\'evy's continuity theorem, $\bfn_2$ converges in distribution to $\normaldist{\bzero}{\Ev{ \bfn_2 \bfn_2^\top }}.$
Since this is true for any finite marginal $\bfn_2$, the Deep GP $f_2(\bfn_1(\cdot))$ converges in distribution to a (single-layer) Gaussian process.
A simple induction extends this to multiple layers.
\end{proof}

\subsection{Comparison to \citet{agrawal2020wide}.}
\citet[][Thm. 8]{agrawal2020wide} also study infinite-width limits of Deep GP, though their analysis is restricted to a sub-class of models.
Specifically, they focus on infinitely-wide neural networks with finite bottleneck layers---a specific class of Deep GP
that they refer to as \emph{bottleneck NNGP}.
As the width of the bottleneck layers grow, these models become neural networks with infinite width in all layers.
Coupling this with the analysis of infinitely-wide neural networks \citep{lee2017deep,matthews2018gaussian}, we have that bottleneck NNGP converge to standard GP in the limit of infinite width.
However, the authors note that not every Deep GP can be expressed by the bottleneck NNGP architecture \citep[][Remark 7]{agrawal2020wide},
and so their analysis is not sufficient to prove that all Deep GP converge to GP.

It is worth considering whether this strategy can be applied to other architectures---i.e. what Deep GP can be reduced to infinite-width neural networks with bottlenecks.
For example, \citet{cutajar2017random} study Deep GP with isotropic covariances.
They convert each GP layer into neural network-like layers using random Fourier features \citep{rahimi2007random}.
However, their model is not exactly equivalent to a neural network, and so the analysis of \citet{agrawal2020wide} does not immediately apply.
Moreover, it is not obvious how to express nonstationary GP as neural network-like architectures with modular width.

Finally, we remark that the strategy of \citet[][Thm. 8]{agrawal2020wide} is in some sense the opposite of what is explored in this paper.
They and others \citep{duvenaud2014avoiding,cutajar2017random,dutordoir2021deep} reduce certain classes of Deep GP to infinitely-wide neural networks with bottlenecks;
conversely, we reduce neural networks to Deep GP.

\section{Proofs of Theorems~\ref{thm:mean_concentration} and \ref{thm:tails}}
\label{sec:proof_tails}

Both \cref{thm:mean_concentration,thm:tails} use the same two-step strategy presented for \cref{lemma:nn_cf}.
We will decompose an expectation using the law of total expectation, and then apply Jensen's inequality.

\newtheorem*{thmpeak}{Theorem~\ref{thm:mean_concentration} (Restated)}
\begin{thmpeak}
  Let $f_L \circ \ldots \circ \bfn_1(\cdot)$ be a zero-mean Deep GP.
  Given a finite set of inputs $\bX = [\bx_1, \ldots, \bx_N]$, define
  $\bfn_\ell = [(f_\ell \circ \ldots \circ \bfn_1(\bx_1)), \ldots, (f_\ell \circ \ldots \circ \bfn_1(\bx_N))]$ for $\ell \in [1, L]$,
  and define $\bK_\text{lim} = \Evover{\bfn_L}{ \bfn_L \bfn_L^\top }$.
  Then,
  $
  p( \bfn_L = \bzero) > \normaldist{\bg = \bzero; \bzero}{ \bK_\text{lim} }.
  $
\end{thmpeak}

\begin{proof}
  We first produce a bound on the characteristic function of $\bfn_L$:
  \begin{align}
    \Evover{ \bfn_L } { \exp( i \bt^\top \bfn_L )}
    &=
    \Evover{ \bF_1 }{ \Evover{\bF_2 \mid \bF_1}{ \ldots \Evover{\bF_{L-1} \mid \bF_{L-2}}{ \Evover{ \bfn_L \mid \bF_{L-1} }{
        \exp( i \bt^\top \bfn_L )
    }}}}
    \tag{law of total expectation}
    \\
    &=
    \Evover{ \bF_1 }{ \Evover{\bF_2 \mid \bF_1}{ \ldots \Evover{ \bF_{L-1} \mid \bF_{L-2} }{
      \exp \left( - \frac 1 2 \bt^\top \bK_L\left( \bF_{L-1}, \bF_{L-1} \right) \bt \right)
    }}}
    \tag{Gaussian characteristic function}
    \\
    &\geq
    \exp\left( -\frac 1 2 \bt^\top
      \Evover{ \bF_1 }{ \Evover{\bF_2 \mid \bF_1}{ \ldots \Evover{ \bF_{L-1} \mid \bF_{L-2} }{
        \bK_L\left( \bF_{L-1}, \bF_{L-1} \right)
      }}}
    \bt \right)
    \tag{Jensen's inequality, strict convexity of $\exp$}
    \\
    &=
    \exp\left( -\frac 1 2 \bt^\top
      \Evover{ \bF_{L-1} }{
        \bK_L\left( \bF_{L-1}, \bF_{L-1} \right)
      }
    \bt \right)
    \tag{law of total expectation}
    \\
    &=
    \exp\left( -\frac 1 2 \bt^\top
      \Evover{ \bfn_{L} }{
        \bfn_L \bfn_L^\top
      }
    \bt \right)
    =
    \exp\left( -\frac 1 2 \bt^\top
      \bK_\text{lim}
    \bt \right)
    =
    \Evover{\bg \sim \normaldist{\bzero}{\bK_\text{lim}}} { \exp( i \bt^\top \bg ) }.
    \label{eqn:cf_bound}
  \end{align}
  Thus, we have
  \begin{align*}
    p( \bfn_L = \bzero)
    &= \int \Evover{ \bfn_L }{ \exp( i \bt^\top \bfn_L ) } \intd \bt
    \\
    &\geq \int \Evover{ \bg \sim \normaldist{\bzero}{\bK_\text{lim}} }{ \exp( i \bt^\top \bg ) } \intd \bt
    \:\: = \:\:  \normaldist{\bg = \bzero; \bzero}{ \bK_\text{lim} },
  \end{align*}
  which completes the proof.
\end{proof}

It is worth noting that the Jensen gap of the characteristic function cascades with depth.
For example, given the 3-layer model $f_3(\bfn_2(\bfn_1(\cdot)))$, we have:
\begin{align*}
  \underbracket{\Evover{\bF_1}{\Evover{\bF_2 \mid \bF_1}{ \exp \left( -\frac 1 2 \bt^\top \bK_3 \bt \right) }}}_{\text{CF of 3-layer Deep GP marginal}}
  &\geq
  \underbracket{\Evover{\bF_1}{ \exp \left( -\frac 1 2 \bt^\top \!\! \Evover{\bF_2 \mid \bF_1}{ \bK_3 } \bt \right)}}_{\text{CF of 2-layer Deep GP marginal}}
  \\
  &\geq
  \underbracket{\exp\left( -\frac 1 2 \bt^\top \!\! \Evover{\bF_1}{\Evover{\bF_2 \mid \bF_1}{ \bK_3 }} \bt \right)}_{\text{CF of $\normaldist{\bzero}{\Evover{\bF_1}{\Evover{\bF_2 \mid \bF_1}{ \bK_3 }}}$}},
\end{align*}
Consequentially, the peaks of deeper models will be sharper than those of shallower models.
This is analogous to the tail effects of depth analyzed in \cref{sec:tails}.

\newtheorem*{thmtails}{Theorem~\ref{thm:tails} (Restated)}
\begin{thmtails}
  Let $\bt \in \reals^N$.
  Using the same setup, notation, and assumptions as \cref{thm:mean_concentration},
  the odd moments of $\bt^\top \bfn_L$ are zero and the even moments larger than $2$ are super-Gaussian,
  i.e. $\Evover{\bfn_L}{ (\bt^\top \bfn_L)^{r} } \geq \Evover{ \bg \sim \normaldist{\bzero}{\bK_\text{lim}} }{ (\bt^\top \bg)^r }$
  for all even $r \geq 4$.
  Moreover, if $k_L(\cdot, \cdot)$ is bounded almost everywhere, the moment generating function $\Evover{\bfn_L}{ \exp( \bt^\top \bfn_L) }$ exists and is similarly super-Gaussian.
\end{thmtails}
\begin{proof}
  We can express the moments of $\bt^\top \bfn_L$ as:
  \begin{align*}
    \Evover{\bfn_L}{ \left( \bt^\top \bfn_L \right)^r }
    &=
    \Evover{ \bF_1 }{ \Evover{\bF_2 \mid \bF_1}{ \ldots \Evover{\bF_{L-1} \mid \bF_{L-2}}{ \Evover{ \bfn_L \mid \bF_{L-1} }{
       \left( \bt^\top \bfn_L \right)^r
    }}}},
  \end{align*}
  and note that the innermost expectation can be simplified to:
  \begin{align*}
    \Evover{\bfn_L \mid \bF_{L-1}}{ \left( \bt^\top \bfn_L \right)^r }
    = \begin{cases}
      0 & r \text{ is odd} \\
      \left( \bt^\top \bK_L(\bF_{L-1}, \bF_{L-1}) \bt \right)^{\frac r 2} (r - 1)!! & r \text{ is even}.
    \end{cases}
    \tag{Gaussian moments}
  \end{align*}
  For odd $r$, note that this implies $\Evover{\bfn_L}{ (\bt^\top \bfn_L)^r } = 0$.
  For even $r \geq 4$, note that $\bt^\top \bK_L(\bF_{L-1}, \bF_{L-1}) \bt \geq 0$ by positive definiteness of kernels,
  and note that $(z)^{r/2}$ is convex for all $z \geq 0$.
  Following the same logic as in the proof for \cref{thm:mean_concentration}, we have:
  \begin{align*}
    \Evover{\bfn_L}{ \left( \bt^\top \bfn_L \right)^r }
    &=
    \Evover{ \bF_1 }{ \Evover{\bF_2 \mid \bF_1}{ \ldots \Evover{\bF_{L-1} \mid \bF_{L-2}}{
        \left( \bt^\top \bK_L(\bF_{L-1}, \bF_{L-1}) \bt \right)^{\frac r 2} (r - 1)!!
    }}}
    \\
    &\geq
    \left( \bt^\top
      \Evover{ \bF_1 }{ \Evover{\bF_2 \mid \bF_1}{ \ldots \Evover{\bF_{L-1} \mid \bF_{L-2}}{
        \bK_L(\bF_{L-1}, \bF_{L-1})
      }}}
    \bt \right)^{\frac r 2} (r - 1)!!
    \tag{Jensen's inequality, strict convexity}
    \\
    &=
    \left( \bt^\top \bK_\text{lim} \bt \right)^{\frac r 2} (r - 1)!!
    =
    \Evover{\normaldist{\bg; \bzero}{\bK_\text{DGP}(\bX, \bX)}}{ (\bt^\top \bg)^r }
  \end{align*}
  A similar proof will show that the moment generating function is similarly super-Gaussian:
  \begin{align}
    \Evover{ \bfn_L } { \exp( \bt^\top \bfn_L )}
    &=
    \Evover{ \bF_{L-1} }{
      \exp \left( \frac 1 2 \bt^\top \bK_L\left( \bF_{L-1}, \bF_{L-1} \right) \bt \right)
    }
    \label{eqn:dgp_mgf_expanded}
    \\
    &=
    \Evover{ \bF_1 }{ \Evover{\bF_2 \mid \bF_1}{ \ldots \Evover{ \bF_{L-1} \mid \bF_{L-2} }{
      \exp \left( \frac 1 2 \bt^\top \bK_L\left( \bF_{L-1}, \bF_{L-1} \right) \bt \right)
    }}}
    \nonumber
    \\
    &\geq
    \exp\left( \frac 1 2 \bt^\top
      \Evover{ \bF_1 }{ \Evover{\bF_2 \mid \bF_1}{ \ldots \Evover{ \bF_{L-1} \mid \bF_{L-2} }{
        \bK_L\left( \bF_{L-1}, \bF_{L-1} \right)
      }}}
    \bt \right)
    \nonumber
    \\
    &=  \exp \left( \frac 1 2 \bt^\top \bK_\text{lim} \bt \right)
    = \Evover{\normaldist{\bg; \bzero}{\bK_\text{lim}}}{ \exp( \bg^\top \bt ) }.
    \nonumber
  \end{align}
  We know that the moment generating function exists because, by assumption, $k_L(\cdot, \cdot)$ is bounded almost everywhere, and thus the integral defined by the expectation in \cref{eqn:dgp_mgf_expanded} is finite.
\end{proof}

\section{Derivation of Deep GP Covariances and Limiting GP Covariances}
\label{sec:limiting_kernels}

Here we derive the prior covariances of various Deep GP architectures,
as well as the covariances of their corresponding infinite width GP limits.

\subsection{RBF + Additive RBF}
First, consider a two layer Deep GP $f_2(\bfn_1(\cdot))$
where the first layer uses a RBF covariance
and the second layer uses a sum of 1-dimensional RBF covariances:
\begin{align*}
  k_1( \bx, \bx' ) &= o_1^2 \exp \left( \frac{- \Vert \bx - \bx' \Vert^2_2}{2 \ell_1^2} \right),
  \\
  k_2( \bfn_1(\bx), \bfn_1(\bx') ) &= \frac{o_2^2}{H_1} \sum_{i=1}^{H_1} \exp \left( - \frac{( f_1^{(i)}(\bx) - f_1^{(i)}(\bx') )^2}{ 2 \ell_2^2 } \right),
\end{align*}
where $H_1$ is the width of the Deep GP, and $o_1$, $\ell_1$, $o_2$, and $\ell_2$ are hyperparameters.
This is the Deep GP architecture most commonly explored in this paper.

To calculate the covariance between $f_2(\bfn_1(\bx))$ and $f_2(\bfn_1(\bx'))$, we first note that
$\tau_i \triangleq f_1^{(i)}(\bx) - f_1^{(i)}(\bx')$ is Gaussian distributed:
\begin{align}
  \tau_i = f_1^{(i)}(\bx) - f_1^{(i)}(\bx') \sim \normaldist{0}{ \sigma^2 },
  \qquad
  \sigma^2 \triangleq k_1(\bx, \bx) + k_1(\bx', \bx') - 2 k_1(\bx, \bx').
  \label{eqn:tau_diff}
\end{align}
The covariance of the Deep GP is therefore given as:
\begin{align*}
  \Evover{\text{RBF} + \text{add-RBF}}{f_2(\bfn_1(\bx)) f_2(\bfn_1(\bx')) }
  &=
  \Evover{\bfn_1(\bx), \bfn_1(\bx')}{k_2(\bfn_1(\bx), \bfn_1(\bx') }
  \\
  &=
  \Evover{\tau_i}{\frac{o_2^2}{H_1} \sum_{i=1}^{H_1} \exp \left( - \frac{\tau_i^2}{ 2 \ell_2^2 } \right)}
  \tag{$\tau_i \triangleq f_1^{(i)}(\bx) - f_1^{(i)}(\bx')$} \nonumber
  \\
  &=
  o^2_2 \Evover{\tau_1}{\exp \left( - \frac{\tau_1^2}{ 2 \ell_2^2 } \right)}
  \tag{$\tau_i$ are i.i.d.}
  \\
  &= \frac{o_2^2}{\sqrt{2 \pi \sigma^2}}
    \int_{-\infty}^{\infty}
      \exp\left( \frac{ -\tau_1^2 }{2 \ell_2^2}  \right)
      \exp\left( \frac{ -\tau_1^2 }{2 \sigma^2}  \right)
    \intd \tau_1
  \nonumber
  \\
  &= \frac{o_2^2}{\sqrt{2 \pi \sigma^2}}
    \int_{-\infty}^{\infty}
      \exp\left( - \tau_1^2 \frac{ \sigma^2 + \ell_2^2}{2 \sigma^2 \ell_2^2} \right)
    \intd \tau_1
  \nonumber
  \\
  &= \frac{o_2^2}{\sqrt{2 \pi \sigma^2}}
     \sqrt{ \frac{2 \pi \sigma^2 \ell_2^2}{ \sigma^2 + \ell_2^2 }}.
  \tag{Gaussian normalizing constant}
\end{align*}
Plugging in $\sigma^2$ from \cref{eqn:tau_diff}, we have
\begin{equation}
  \Evover{\text{RBF} + \text{add-RBF}}{f_2(\bfn_1(\bx)) f_2(\bfn_1(\bx')) } = o_2^2 \left( 1 + \frac{2 o_1^2 \left( 1 - \exp\left(
    -\frac{\Vert \bx - \bx' \Vert^2_2}{2 \ell_1^2} \right ) \right)}{\ell_2^2}
  \right)^{-1/2}.
  \label{eqn:add_rbf_dgp}
\end{equation}
Note that this covariance is the same, regardless of the Deep GP width $H_1$.
Therefore, it is also the covariance of the infinite width GP limit.
More generally, if we replace the first RBF covariance with an arbitrary covariance function $k_1(\cdot, \cdot)$
we have:
\begin{equation}
  \Evover{k_1 + \text{add-RBF}}{f_2(\bfn_1(\bx)) f_2(\bfn_1(\bx')) } = o_2^2 \left( 1 + \frac%
    {k_1(\bx, \bx) + k_1(\bx', \bx') - 2 k_1(\bx, \bx')}%
    {\ell_2^2}
  \right)^{-1/2}.
  \label{eqn:add_rbf_dgp_general}
\end{equation}
A similar derivation can be found in \citep{lu2020interpretable}.

\subsection{RBF + (Non-Additive) RBF}
Now consider a two layer Deep GP $f_2(\bfn_1(\cdot))$
where the first and second layers both use \emph{non-additive} RBF covariance functions:
\begin{align*}
  k_1( \bx, \bx' ) &= o_1^2 \exp \left( \frac{- \Vert \bx - \bx' \Vert^2_2}{2 \ell_1^2} \right),
  \\
  k_2( \bfn_1(\bx), \bfn_1(\bx') ) &= o_2^2 \exp \left(
    - \frac{\left\Vert \bfn_1(\bx) - \bfn_1(\bx') \right\Vert_2^2}%
    { 2 H_1 \ell_2^2 }
  \right),
\end{align*}
where the $1 / H_1$ factor is included to reduce the impact of the dimensionality of $\bfn_1(\cdot)$.
This is a very common Deep GP architecture \citep[e.g.][]{damianou2013deep,bui2016deep,cutajar2017random,salimbeni2017doubly},
and it is the architecture in the \cref{sec:posterior} example.
Crucially, the RBF kernel decomposes as a product across its dimensions:
\begin{align*}
  k_2( \bfn_1(\bx), \bfn_1(\bx') ) =
  o_2^2 \exp \left(
    - \frac{\left\Vert \bfn_1(\bx) - \bfn_1(\bx') \right\Vert_2^2}%
    { 2 H_1 \ell_2^2 }
  \right)
  =
  o_2^2 \prod_{i=1}^{H_1} \exp \left(
    - \frac{\left( f_1^{(i)}(\bx) - f_1^{(i)}(\bx') \right)^2}%
    { 2 H_1 \ell_2^2 }
  \right)
\end{align*}
Since the $f_1^{(i)}(\cdot)$ are independent, we have:
\begin{align*}
  \Evover{\text{RBF} + \text{RBF}}{f_2(\bfn_1(\bx)) f_2(\bfn_1(\bx')) }
  &=
  \Evover{\bfn_1(\bx), \bfn_1(\bx')}{
    o_2^2 \prod_{i=1}^{H_1} \exp \left(
      - \frac{\left( f_1^{(i)}(\bx) - f_1^{(i)}(\bx') \right)^2}%
      { 2 H_1 \ell_2^2 }
    \right)
  }
  \\
  &=
  o_2^2 \prod_{i=1}^{H_1} \Evover{f_1^{(i)}(\bx), f_1^{(i)}(\bx')}{
    \exp \left(
      - \frac{\left( f_1^{(i)}(\bx) - f_1^{(i)}(\bx') \right)^2}%
      { 2 H_1 \ell_2^2 }
    \right)
  }
  \\
  &=
  o_2^2 \prod_{i=1}^{H_1}
  \Evover{\text{RBF} + \text{add-RBF}}{f_2 \left(\frac{f^{(1)}_1(\bx)}{\sqrt{H_1}}\right) f_2 \left(\frac{f^{(1)}_1(\bx')}{\sqrt{H_1}}\right) },
\end{align*}
Plugging in \cref{eqn:add_rbf_dgp_general}, we have
\begin{align}
  \Evover{\text{RBF} + \text{RBF}}{f_2(\bfn_1(\bx)) f_2(\bfn_1(\bx')) }
  &= o_2^2 \left( 1 + \frac%
    {k_1(\bx, \bx) + k_1(\bx', \bx') - 2 k_1(\bx, \bx')}%
    {H_1 \ell_2^2}
  \right)^{-H_1/2}.
  \nonumber
  \\
  &= o_2^2 \left( 1 + \frac%
    {2o_1^2 \left( 1 - \exp\left(
      -\frac{\Vert \bx - \bx' \Vert_2^2}{2 \ell_1^2}
    \right) \right)}%
    {H_1 \ell_2^2}
  \right)^{-H_1/2}.
\end{align}
In the limit as $H_1 \to \infty$, this second moment becomes:
\begin{align}
  \lim_{H_1 \to \infty}
  \Evover{\text{RBF} + \text{RBF}}{f_2(\bfn_1(\bx)) f_2(\bfn_1(\bx')) }
  &= o_2^2 \exp \left( -\frac{ k_1(\bx, \bx) + k_1(\bx', \bx') - 2 k_1(\bx, \bx') } { 2 \ell_2^2 } \right).
  \nonumber
  \\
  &= o_2^2 \exp \left( \frac{o_1^2}{\ell_2^2} \exp \left( -\frac{\Vert \bx - \bx' \Vert_2^2}{2 \ell^2_1} \right) - 1 \right).
  \label{eqn:rbf_dgp_infty}
\end{align}

\subsection{RBF + Additive RBF + Additive RBF}
\label{sec:rbf3l}

Now consider the three layer Deep GP $f_3(\bfn_2(\bfn_1(\cdot)))$,
where the first layer uses an RBF covariance and the other layers use sums of 1-dimensional RBF covariances.
\begin{align*}
  k_1( \bx, \bx' ) &= o_1^2 \exp \left( \frac{- \Vert \bx - \bx' \Vert^2_2}{2 \ell_1^2} \right),
  \\
  k_2( \bfn_1(\bx), \bfn_1(\bx') ) &= \frac{o_2^2}{H_1} \sum_{i=1}^{H_1} \exp \left( - \frac{( f_1^{(i)}(\bx) - f_1^{(i)}(\bx') )^2}{ 2 \ell_2^2 } \right),
  \\
  k_3( \bfn_2(\bfn_1(\bx)), \bfn_2(\bfn_1(\bx')) ) &= \frac{o_3^2}{H_2} \sum_{i=1}^{H_2} \exp \left( - \frac{\left( f_2^{(i)}(\bfn_1(\bx)) - f_2^{(i)}(\bfn_1(\bx')) \right)^2}{ 2 \ell_3^2 } \right),
\end{align*}
where $H_1$, $H_2$ are the widths of the first and second layers,
and $o_1$, $\ell_1$, $o_2$, $\ell_2$, $o_3$ and $\ell_3$ are hyperparameters.
We use this architecture in \cref{sec:tails,sec:experiments_bayesian}.

Unfortunately, it is intractable to compute the second moment of this Deep GP in closed form.
To see why this is the case, note that:
\begin{align}
  \Evover{\text{RBF} + \text{add-RBF} + \text{add-RBF}}{f_3(\bfn_2(\bfn_1(\bx))) \:\: f_3(\bfn_2(\bfn_1(\bx'))) }
  \:
  =
  \:
  \Evover{\bfn_2(\bfn_1(\bx)), \bfn_2(\bfn_1(\bx'))}{k_3(\bfn_2(\bfn_1(\bx)), \: \bfn_2(\bfn_1(\bx'))) }
  \label{eqn:3layer_covar}
\end{align}
In other words, computing the covariance requires taking the expectation over the Deep GP marginal $\bfn_2(\bfn_1(\bx)), \bfn_2(\bfn_1(\bx'))$ which is intractable to compute.
We do note that we can approximate this marginal with Gauss-Hermite quadrature if $H_1$ is sufficiently small.
Moreover, unlike the 2-layer case, the width $H_1$ affects the second moment of this 3-layer Deep GP.
This is because changing $H_1$ changes the marginal distribution $\bfn_2(\bfn_1(\bx)), \bfn_2(\bfn_1(\bx'))$, which ultimately impacts the expectation in \cref{eqn:3layer_covar}.
Changing the value of $H_2$ does not affect the covariance by linearity of expectation, assuming that we hold $H_1$ constant.

As $H_1 \to \infty$ and $\bfn_2(\bfn_1(\cdot))$ converges to a Gaussian process,
the expectation in \cref{eqn:3layer_covar} becomes tractable again.
$f_3(\bfn_2(\bfn_1(\cdot)))$ effectively becomes a 2-layer Deep GP,
where the first layer has covariance given by \cref{eqn:add_rbf_dgp}
and the second layer is the sum of 1-dimensional RBF covariances.
Thus, combining \cref{eqn:add_rbf_dgp} and \cref{eqn:add_rbf_dgp_general},
we can compute the covariance of the limiting GP:

\begin{align}
  &\phantom{=} \lim_{H_2 \to \infty}{H_1 \to \infty} \Evover{\text{RBF} + \text{add-RBF} + \text{add-RBF}}{f_3(\bfn_2(\bfn_1(\bx))) \:\: f_3(\bfn_2(\bfn_1(\bx'))) }
  \nonumber
  \:
  \\
  &= o_3^2 \left( 1 + \frac{
    2 o_2 \left( 1 -  \left( 1 + \frac{2 o_1^2 \left( 1 - \exp\left(
      -\frac{\Vert \bx - \bx' \Vert^2_2}{2 \ell_1^2} \right ) \right)}{\ell_2^2}
    \right)^{-1/2} \right)
  }{\ell_3^2} \right)^{-1/2}
  \nonumber
\end{align}

A similar derivation can be found in \citep{lu2020interpretable}.

\subsection{Neural networks}
The prior second moment of a neural network with ReLU activations and a single hidden layer
(i.e. the construction in Eq.~\ref{eqn:2layer_nn})
is given by the arc-cosine kernel \citep{cho2009kernel}:
\begin{equation}
  \Evover{\text{2-layer NN}}{
    f_2(\bfn_1(\bx)) f_2(\bfn_1(\bx'))
  } = \beta^2 + \frac{1}{2 \pi} \Vert \bx \Vert \: \Vert \bx' \Vert \left(
    \sin(\theta) + ( \pi - \theta) \cos(\theta)
  \right),
\end{equation}
where $\theta = \cos^{-1} \left( (\bx^\top \bx') / (\Vert \bx \Vert \: \Vert \bx' \Vert) \right)$.
Note that this covariance is constant regardless of $H_1$.

As with the 3-layer RBF Deep GP, we cannot compute the prior second moment of deeper neural networks in closed form.
Moreover, once neural networks have more than 1 hidden layer, then the width of hidden layers affects the covariance.
Nevertheless, we can compute the limiting infinite width covariance using the recursive formula defined in \citep{cho2009kernel,lee2017deep}.

\section{Experimental Details}
\label{sec:experimental_details}

The experiments are implemented in PyTorch \citep{paszke2019pytorch}, supplemented by the Pyro \citep{bingham2019pyro} and GPyTorch \citep{gardner2018gpytorch} libraries -- all of which are open source.
We ran them on a cluster with GTX1080 and GTX2080 GPU,
and we estimate that we used $\approx 1,\!000$ hours of GPU compute time.
The largest experiments (CIFAR10 ResNets with maximum width) require $48 GB$ of GPU memory;
all other experiments only require $\leq 11 GB$ of memory.

\paragraph{Datasets.}
The datasets for the regression experiments are from the UCI repository \citep{asuncion2007uci}.
Unless otherwise stated, we split these datasets into $75\%$ training data,
$15\%$ test data, and $10\%$ validation data.
For larger datasets, we subsample the training dataset to a maximum of $N=1,\!000$ data points.
All input features are normalized to be between $-1$ and $1$,
and the $\by$ values are z-scored to have $0$ mean and unit variance.
For the non-Bayesian neural network experiments, we use the MNIST \citep{lecun1998gradient} and CIFAR10 \cite{krizhevsky2009learning} datasets.
We z-score the inputs so that each channel has $0$ mean and unit variance.
We use the standard $10,\!000$ data point test sets, and subsample the remaining data for training.
Models are trained without data augmentation.

\paragraph{Deep GP models.}
All Deep GP models use GP layers with zero prior mean.
We perform inference without making any scalable approximations,
though we do add a constant diagonal of $10^{-4}$ to all prior covariances for stability.
We perform inference using the NUTS sampler \citep{hoffman2014no} implemented in Pyro \citep{bingham2019pyro},
using $500$ warmup steps, drawing $500$ samples, a target acceptance probability of $0.8$, an initial learning rate of $0.1$,
and a maximum tree depth of $10$.
To improve inference, we infer the ``whitened'' latent variables $\bL_1^{-1} \bF_1$ and $\bL_2^{-1} \bfn_2$,
where $\bL_1$ and $\bL_2$ are the Cholesky factors of $\bK_1(\bX, \bX)$ and $\bK_2(\bF_1, \bF_1)$ respectively.
For all width $\geq 2$ Deep GP, we initialize the latent variables by running $1,\!000$ steps of Adam \citep{kingma2014adam}
with learning rate $0.01$
on the maximum a posteriori Deep GP objective.
Because width-$1$ Deep GP inference is more challenging (see the control experiment in \cref{sec:additional_results}), we instead initialize the latent variables of these models from the mean of a doubly stochastic variational Deep GP \citep{salimbeni2017doubly},
where we use $300$ inducing points per layer, $10$ function samples, and a minibatch size of $128$.
We optimize the variational Deep GP with Adam for $2,\!000$ iterations, using an initial learning rate of $0.01$, dropping it by a factor of $10$ after $50\%$ and $75\%$ of training.

For each Deep GP model, we use hyperparameters that maximize the log marginal likelihood of the corresponding limiting GP.
To find these hyperparameters, we perform $100$ iterations of Adam on the limiting GP using a learning rate of $0.1$, initializing all covariance hyperparameters to $1$ and initializing the likelihood observational noise to $0.2$.

\paragraph{Bayesian neural network models.}
We train the Bayesian neural networks in a very similar manner.
However, we perform $2,\!000$ warmup steps and draw $1,\!000$ samples using NUTS.
Again, we use the same hyperparameters as the optimized limiting GP.

\paragraph{MNIST experiments (Non-Bayesian neural networks).}
We train 3-layer (2 hidden-layer) MLP with an L2 regularization constant of $\in \{ 10^{-5}, 10^{-6}, 10^{-7}, 10^{-8} \}$,
which corresponds to per-parameter priors of $\{ \normaldist{0}{2}, \normaldist{0}{20}, \normaldist{0}{200}, \normaldist{0}{2000} \}$ when $N=50,\!000$.
Following \cref{eqn:2layer_nn}, the output of layer $\ell$ is scaled by a factor of $1 / \sqrt{H_{\ell-1}}$, where ${H_{\ell-1}}$ is the width of layer $\ell - 1$.
We train the models for $20,\!000$ iterations using the Adam optimizer.
Following \citet{lee2017deep}, we perform a random search over the learning rate and batch size parameters.
More specifically, we randomly choose $10$ learning rate/batch size tuples from $[0.001, 0.2] \times \{ 16, 32, 64, 128, 256 \}$,
and select the hyperparameters that generate the best accuracy on the withheld training data.
Note that these hyperparameters do not affect the prior of the model's Bayesian analog; they only impact the optimization dynamics.
We drop the learning rate by a factor of $10$ after $50\%$ and $75\%$ of training.

\paragraph{CIFAR10 experiments (Non-Bayesian neural networks).}
We do not perform the $1 / \sqrt{H_{\ell - 1}}$ scaling on the ResNet models, as this scaling is undone by batch normalization.
Models are optimized using SGD with an initial learning rate of $0.1$,
following the same schedule as the MNIST models.
The L2 regularization constant ($10^{-4}$) corresponds to a per-parameter prior of $\normaldist{0}{0.2}$.
Note that these hyperparameters exactly match those suggested by \citet{he2016deep} and \citet{zagoruyko2016wide}.
We train each model for $40,\!000$ iterations with a minibatch of $256$.

\paragraph{Effect of depth experiments in \cref{sec:experiments_bayesian}.}
In these experiments, our goal is to investigate the effects of depth while controlling for the first and second moments of the Deep GP models.
To that end, we construct a 3-layer Deep GP, 2-layer Deep GP, and a single-layer GP all with zero mean and the same prior covariance.
The 3-layer Deep GP uses a RBF covariance in the first layer, and sums of 1-dimensional RBF kernels in the other two layers.
We set the widths to be $H_1=2$ and $H_2=8$.
The 2-layer Deep GP uses a RBF covariance in the first layer with a width of $H_1=2$, while the second layer uses the following covariance:
\begin{equation}
  o_3^2 \left( 1 + \frac{2 o_2^2 \left( 1 - \frac{1}{H_1} \sum_{i=1}^{H_1} \exp\left(
    -\left( f^{(i)}_1(\bx) - f^{(i)}_1(\bx') \right)^2/ (2 \ell_2) \right ) \right)}{\ell_3^2}
  \right)^{-1/2}.
\end{equation}
For the single layer GP, we compute the covariance of the 3-layer model using the formula in \cref{eqn:3layer_covar}.
We approximate the marginal distribution $p(\bfn_2(\bfn_1(\bx)), \bfn_2(\bfn_1(\bx'))$ using Gauss-Hermite quadrature with 11 nodes.
We empirically confirm that the single-layer, 2-layer, and 3-layer models have the same prior covariance.

\paragraph{Visualizing $N=2$ marginal densities in \cref{sec:tails}.}
The 3-layer Deep GP uses a RBF covariance in the first layer, and sums of 1-dimensional RBF kernels in the other two layers.
We vary both widths $H_1$ and $H_2$ simultaneously, and we set all hyperparameters to $1$.
The 2-layer Deep GP are designed to match the prior covariance of the $\text{width}=1$ 3-layer model.
To that end, the first layer uses an RBF covariance, and the second layer uses the following covariance:
\begin{equation}
  \left( 1 + 2 \left( 1 - \sum_{i=1}^{H_1} \exp\left(
    \frac{-\left( f^{(i)}_1(\bx) - f^{(i)}_1(\bx') \right)^2}{ 2 } \right ) \right)
  \right)^{-1/2}.
\end{equation}
Again, we empirically verify that these models have the same  prior covariance.
We approximate the marginal densities at $400$ evenly spaced grid points on $\by \in [-3, 3] \times [3, 3]$
using Gauss-Hermite quadrature with 7 nodes.

\end{document}